\documentclass{article}
\pdfoutput=1 
\usepackage{arxiv}
\renewcommand{\undertitle}{Under consideration for publication in Theory and Practice of Logic Programming (TPLP)}
\renewcommand{\headeright}{Preprint}
\renewcommand{\shorttitle}{Foundations of MT-PDCL} 
\usepackage{amsmath}
\usepackage{amssymb}
\usepackage{amsthm}
\usepackage{graphicx}
\usepackage[T1]{fontenc}
\usepackage{tabularx}

\newtheorem{proposition}{Proposition}
\newtheorem{theorem}{Theorem}
\newtheorem{definition}{Definition}

\title{Foundations of MT-PDCL: Measure-Theoretic Probabilistic Definite Clause Logic}
\author{
  Costin B\u{a}dic\u{a} \\
  Department of Computers and Information Technology \\
  University of Craiova \\
  Craiova, Romania \\
  \texttt{costin.badica@edu.ucv.ro} 
  \And
  Amelia B\u{a}dic\u{a} \\
  Department of Business Informatics \\
  University of Craiova \\
  Craiova, Romania \\
  \texttt{amelia.badica@edu.ucv.ro}
}
\date{July 2026}

\begin{document}

\maketitle

\begin{abstract}
Standard probabilistic logic programming frameworks typically rely on grounding logic programs into discrete propositional representations. This operational requirement restricts exact inference to finite domains and discrete probability distributions. In this paper, we introduce Measure-Theoretic Probabilistic Definite Clause Logic (MT-PDCL), a generalized foundational framework that eliminates this finite-domain restriction. By explicitly defining stochastic variables over bounded index domains and equipping the interpretation space with standard Borel $\sigma$-algebras, MT-PDCL allows logical variables to operate natively over continuous measurable spaces. Building on Continuous Distribution Semantics, MT-PDCL models probabilistic rules as mutually independent causal events. However, rather than aggregating these derivations via finite boolean circuits, declarative entailment is formally defined through exact Lebesgue integration over the continuous measure space. We introduce a continuous immediate consequence operator that unifies the integration of continuous prior distributions with the evaluation of exact continuous observations. We demonstrate that this approach replaces the combinatorial bottleneck of discrete grounding with exact, algebraic, and structurally differentiable inference. While this transition trades discrete combinatorics for the geometric curse of dimensionality, it achieves the expressive power of continuous probabilistic models while preserving the pure declarative syntax of definite clause logic.
\end{abstract}

\section{Introduction}

The foundational approach of assigning probabilities to logical sentences via probability measures over possible worlds was established by Nilsson \cite{nilsson1986probabilistic}. Building on this, probabilistic logic programming languages such as PRISM \cite{sato1997prism} and ProbLog \cite{deraedt2007problog} introduced distribution semantics to evaluate the probability of logical queries. This approach assumes all base probabilistic facts are mutually independent, thereby defining a single, explicit probability distribution over all possible worlds. However, calculating exact probabilities across these worlds creates a significant operational bottleneck. To evaluate a derived query, the system must sum the probabilities of the specific discrete states where the query holds true. This operation forces the system to ground the first-order logic programs into discrete propositional representations, such as Binary Decision Diagrams (BDDs) \cite{bryant1986graph} or Sentential Decision Diagrams (SDDs) \cite{darwiche2011sdd}. This reliance on propositionalization inherently limits exact inference to finite domains and discrete probability distributions.

Recent advancements in neural-symbolic artificial intelligence have introduced frameworks that integrate differentiable components with logic programming, such as DeepProbLog \cite{manhaeve2021deepproblog} and TensorLog \cite{cohen2020tensorlog}. While these systems successfully process continuous representations, such as neural network parameters or feature tensors, their underlying logical semantics typically remain tied to finite domains. For instance, DeepProbLog evaluates neural network outputs as probabilistic facts, but its logical inference engine still relies on discrete propositionalization.

There is a critical need for a probabilistic logic framework that allows logical variables to operate natively over continuous spaces, completely bypassing the need for finite boolean grounding, while remaining fully differentiable for contemporary deep learning integration.

To address this gap, we introduce Measure-Theoretic Probabilistic Definite Clause Logic (MT-PDCL). This paper makes the following primary contributions:
\begin{itemize}
    \item \textbf{Continuous Declarative Semantics:} Equipping the interpretation space and explicitly bounded stochastic index domains with standard Borel $\sigma$-algebras \cite{kechris1995classical} removes the finite-domain restriction. Logical variables act as deterministic placeholders that natively evaluate over measurable continuous spaces (e.g., bounded subsets of $\mathbb{R}^n$).
    \item \textbf{Continuous Distribution Semantics:} Extending classical distribution semantics to infinite domains, MT-PDCL defines possible worlds via independent continuous variables and probabilistic causal events. Declarative entailment is formally defined as the exact Lebesgue integral of the logically valid possible worlds, evaluated over the geometric boundaries defined by the logic program.
    \item \textbf{Continuous Consequence Operator:} We generalize the immediate consequence operator ($\mathbb{T}_{\mathcal{K}}$) to operate algebraically over continuous probability densities. It utilizes exact integration for continuous marginalization and a continuous Noisy-OR formulation to aggregate disjunctive derivations, mathematically bypassing discrete grounding.
    \item \textbf{Differentiable Neural-Symbolic Foundation:} We formally prove that this continuous operator converges to a well-defined least fixed point and is everywhere differentiable, providing a structurally transparent, non-negative gradient signal that ensures stable routing for neural-symbolic deep learning.
\end{itemize}

The remainder of this paper is structured as follows. Section \ref{sec:syntax} defines the formal syntax of MT-PDCL. Section \ref{sec:semantics} establishes the exact measure-theoretic semantics, building from continuous variable domains and index bounds to formal declarative entailment and continuous query semantics. Section \ref{sec:examples_declarative} provides concrete illustrative examples of this declarative entailment. Section \ref{sec:operational_semantics} introduces the continuous immediate consequence operator and proves its theoretical properties, including fixed-point convergence, operational completeness bounds, and neural-symbolic differentiability. Section \ref{sec:related_work} contextualizes our approach against existing probabilistic and neural-symbolic frameworks. Section \ref{sec:execution_traces} evaluates the dynamic consequence operator through concrete continuous execution traces, comparing these examples against existing approaches to highlight specific operational limitations. Section \ref{sec:limitations} outlines the theoretical boundaries of analytical computability and tractability, and Section \ref{sec:conclusion} concludes.

\section{Syntax of MT-PDCL}\label{sec:syntax}
To formalize MT-PDCL, we define the recursive syntax of its terms, predicates, and rules, ensuring a strict separation between logical reasoning and continuous probability generation. 

Let $\mathcal{L}(\mathcal{C}, \mathcal{R})$ be a first-order language where $\mathcal{C}$ is a finite set of constant symbols and $\mathcal{R}$ is a finite set of predicate symbols. To distinguish user-defined logical relations from standard mathematical operators and continuous probability distributions, $\mathcal{R}$ is partitioned into three disjoint subsets: $\mathcal{R} = \mathcal{R}_u \cup \mathcal{O} \cup \mathcal{R}_s$. 
\begin{itemize}
    \item $\mathcal{R}_u$ contains the user-defined logical relation symbols.
    \item $\mathcal{O}$ contains the built-in base predicate symbols representing standard algebraic and set-theoretic operators (e.g., $<, \le, =, \in$).
    \item $\mathcal{R}_s$ contains the stochastic predicate symbols, which act as the interface to the continuous environment.
\end{itemize}

Following standard logic programming semantics, variables are scoped locally to their clauses. We assume the program is \textit{standardized apart}, meaning all variables are renamed to be unique across different clauses. Let $\mathcal{V}$ represent this global finite set of uniquely named logical variables. In MT-PDCL, these variables are deterministic placeholders used for standard unification; they are not random variables themselves.

\textbf{Terms:} A term $t \in \mathcal{T}$ is either a logical variable $X \in \mathcal{V}$ or a constant symbol $c \in \mathcal{C}$. To support mixed continuous and discrete domains natively, the set of constants $\mathcal{C}$ includes explicit numeric values (e.g., $30.0$, $4.2$), categorical symbols (e.g., \texttt{severe\_hot}, \texttt{critical}), as well as mathematical intervals (e.g., $[0, 10]$) and finite sets (e.g., $\{1, 2, 3\}$). This allows built-in operators to act directly on geometric boundaries and discrete collections. Thus, the set of all terms is $\mathcal{T} = \mathcal{C} \cup \mathcal{V}$.

\textbf{Atomic Formulas (User-Defined Predicates):} An atomic formula is constructed using a user-defined relation symbol $r \in \mathcal{R}_u$ of arity $n$, applied to a tuple of terms: $r(t_1, \dots, t_n)$.

\textbf{Base Predicates (Built-in Operators):} A base predicate is constructed using a built-in mathematical operator $o \in \mathcal{O}$ applied to a tuple of terms, typically written in standard infix notation (e.g., $t_1 < t_2$, $t_1 \le t_2$, $t_1 = t_2$, or $t \in [2,5]$).

\textbf{Stochastic Predicates:} A stochastic predicate $r \in \mathcal{R}_s$ queries the continuous environment. To ensure mathematically valid integration, its arguments are partitioned into deterministic indices (inputs) and generative variables (outputs) using a semicolon: $r(I_1, \dots, I_k ; X_1, \dots, X_m)$. The indices $\mathbf{I}$ must be bound to ground terms prior to evaluation, serving as the unique identity key for the independent random draw. The variables $\mathbf{X}$ must remain unbound until the measure space generates the continuous constants.

\textbf{Rule Bodies ($B$):} A rule body $B$ is a conjunction of atomic formulas, base predicates, and stochastic predicates, defined recursively:
\begin{itemize}
    \item Any atomic formula, base predicate, or stochastic predicate is a valid body $B$.
    \item If $B_1$ and $B_2$ are valid bodies, their conjunction $B_1 \wedge B_2$ is a valid body.
\end{itemize}

Following standard definite clause logic, a rule consists of one positive literal in the head and a conjunction of positive literals in the body. Consequently, negation ($\neg$) is prohibited in both the head and the body, and disjunction ($\lor$) is prohibited in the head. Furthermore, stochastic predicates from $\mathcal{R}_s$ and built-in operators from $\mathcal{O}$ can only appear in the rule body; they are prohibited from appearing in the head.

\textbf{Clauses ($\Phi$):} An MT-PDCL logic program $\Phi$ is a finite set of clauses. Let $H$ be an atomic formula (the rule head) constructed from $\mathcal{R}_u$, and $B$ be a valid rule body. A probabilistic clause takes one of the following exact forms, where $p \in [0, 1]$ is a probability parameter:
\begin{itemize}
    \item \textbf{Probabilistic Rule:} $p :: H \leftarrow B$
    \item \textbf{Probabilistic Fact:} $p :: H$
\end{itemize}
Standard deterministic rules ($H \leftarrow B$) and facts ($H$) are supported as syntactic sugar for probabilistic clauses where the parameter is defined as $p = 1.0$.

\textbf{The Knowledge Base ($\mathcal{K}$):} 
To enforce the separation between logical reasoning and continuous probability, an MT-PDCL Knowledge Base is defined as a tuple $\mathcal{K} = \langle \mathcal{M}_s, \Phi \rangle$. 
\begin{itemize}
    \item $\Phi$ is the logic program containing the definite clauses.
    \item $\mathcal{M}_s$ is the Stochastic Measure Mapping. Mathematically, it is a function that maps a stochastic predicate $r \in \mathcal{R}_s$ and a ground index tuple $\mathbf{i}$ to a specific probability measure $\mathcal{D}$ over the domain. Because measure theory natively unifies continuous and discrete probabilities, $\mathcal{D}$ represents either a continuous probability density (e.g., Normal, Uniform) or a discrete probability mass function (e.g., Categorical, Bernoulli, Poisson). Syntactically, this mapping is declared in the program using the form $r(\mathbf{I}; \mathbf{X}) \sim \mathcal{D}$. To mathematically bind the domain of continuous indices, this declaration can optionally include an explicit geometric bound, written as $r(\mathbf{I}; \mathbf{X}) \sim \mathcal{D} \quad \textbf{over} \quad \mathbf{I} \in \Omega_I$.
\end{itemize}

\textbf{Queries ($Q$):} A query $Q$ is a goal clause consisting of a finite conjunction of literals. Syntactically, it is written as a headless clause: 
$$ \leftarrow L_1 \land L_2 \land \dots \land L_k $$
where each literal $L_i$ can be a user-defined atomic formula, a built-in base predicate, or a stochastic predicate. The conjunction forming the query is therefore structurally identical to a valid rule body $B$. Let $\mathbf{V} = \langle V_1, \dots, V_n \rangle$ be the tuple of all free logical variables occurring in $Q$. In standard logical terms, these free variables are implicitly existentially quantified.

\section{Continuous Distribution Semantics}\label{sec:semantics}

\subsection{Interpretations over Continuous Domains}\label{sec:inter_pw}
To rigorously define probability measures and integration, the universal domain of discourse $D$ cannot be an arbitrary set. We require $D$ to be a measurable space $(D, \mathcal{F}_D)$, where $\mathcal{F}_D$ is a standard Borel $\sigma$-algebra over $D$. To natively support both continuous parameters and discrete categorical properties, $D$ is constructed as the disjoint union of a continuous space (e.g., $\mathbb{R}^n$) and a countable discrete set of semantic values (denoted $D_{discrete}$, containing categorical strings and natural numbers $\mathbb{N}$). This construction satisfies the mathematical requirements for $D$, as the disjoint union of a standard Borel space and a countable discrete set remains a standard Borel space \cite{kechris1995classical}. 

An interpretation $\mathcal{I}$ of $\mathcal{L}$ in the measurable domain $D$ is defined by two mappings:
\begin{itemize}
    \item An interpretation function for constant symbols $\phi_{\mathcal{I}}: \mathcal{C} \rightarrow D$, where $\phi_{\mathcal{I}}(c)$ is the interpretation of the constant $c \in \mathcal{C}$. To natively support continuous logic, MT-PDCL permits numeric values (e.g., $120.5$) to act directly as constants in the syntax, for which $\phi_{\mathcal{I}}$ acts as the identity function.
    \item An interpretation function $\pi_{\mathcal{I}}(r): D^n \rightarrow \{1, 0\}$ for any predicate symbol $r \in \mathcal{R}_u \cup \mathcal{O}$ of arity $n$. For a user-defined relation $r \in \mathcal{R}_u$, $\pi_{\mathcal{I}}(r)$ acts as an indicator function. For a built-in base predicate $r \in \mathcal{O}$ (such as $<, \le, \in$), $\pi_{\mathcal{I}}(r)$ is universally fixed across all interpretations to its standard mathematical definition.
\end{itemize}

This mapping establishes the truth values for \textbf{ground atomic formulas}. Classical logical satisfaction for a specific interpretation $\mathcal{I}$ evaluating a ground atom is defined deterministically as:
$$ \mathcal{I} \models r(c_1, \dots, c_n) \text{ if and only if } \pi_{\mathcal{I}}(r)(\phi_{\mathcal{I}}(c_1), \dots, \phi_{\mathcal{I}}(c_n)) = 1 $$

Logic programs do not explicitly enumerate all facts; instead, they generate them dynamically using universally quantified rules and variables. This allows the logical satisfaction of complex formulas to build compositionally upon these base atoms. Furthermore, variables in MT-PDCL evaluate natively over continuous domains without discrete grounding. The subsequent subsections detail the exact mechanics for resolving these rules, bindings, and queries within a possible world.

To compute valid probabilities, we impose a strict topological requirement on these interpretations: the subset of $D^n$ for which $\pi_{\mathcal{I}}(r)$ evaluates to $1$ must be a valid, measurable set within the product $\sigma$-algebra $\mathcal{F}_D^{\otimes n}$. This ensures that geometric regions defined by logical rules inherently possess a well-defined mathematical volume.

\subsection{The Stochastic Measure Space}
\label{sec:continuous_variables}
In standard probabilistic logic programming, it is often erroneously implied that logical variables themselves hold probabilities. In MT-PDCL, we enforce a strict separation. Logical variables in $\mathcal{V}$ are deterministic, syntactic placeholders used for unification. All probabilistic uncertainty including both continuous and discrete is isolated within the stochastic measure space, generated by the mapping $\mathcal{M}_s$.

Recall that a stochastic predicate $r \in \mathcal{R}_s$ is partitioned into deterministic indices $\mathbf{I}$ and stochastic outputs $\mathbf{X}$ of arity $m$. Let $\mathcal{I}_r$ be the set of all valid ground instantiations of the index tuple $\mathbf{I}$ over the domain $D$, and let $\Delta(D^m)$ be the space of all valid probability measures over the $m$-dimensional output space.

The mapping $\mathcal{M}_s : \mathcal{R}_s \times \mathcal{I}_r \rightarrow \Delta(D^m)$ formally assigns a probability measure to every unique ground index of a stochastic predicate. For every stochastic predicate $r$ and every unique ground index $\mathbf{i} \in \mathcal{I}_r$, the output of the mapping $\mathcal{M}_s(r, \mathbf{i})$ defines a mathematically independent random variable (or tuple of $m$ variables). While we primarily refer to these as continuous variables, measure theory treats discrete categorical variables as a direct special case using probability measures concentrated on specific points (Dirac measures). Thus, a stochastic predicate can natively generate either continuous coordinates or discrete categorical states distributed according to its assigned measure.

While the set of ground indices $\mathcal{I}_r$ may be countably or uncountably infinite, the mapping $\mathcal{M}_s$ is defined compactly in the knowledge base using parameterized probability distributions. The parameters of the assigned distribution $\mathcal{D}$ can be deterministic functions of the index tuple $\mathbf{I}$, denoted syntactically as $r(\mathbf{I}; \mathbf{X}) \sim \mathcal{D}(\theta(\mathbf{I}))$. For example, a continuous sensor reading whose mean drifts over time $T$ can be compactly declared as $\text{reading}(T; X) \sim \text{Normal}(\mu = 0.5 \cdot T, \sigma = 1.0)$. Similarly, a discrete weather generator can be declared as $\text{weather}(T; W) \sim \text{Categorical}(\{\text{sunny}: 0.7, \text{rain}: 0.3\})$. Likewise, a hybrid mapping where a continuous index directly parameterizes a discrete countable distribution can be declared as $\text{fault\_count}(T; K) \sim \text{Poisson}(\lambda = 2.0 \cdot T)$. Mathematically, this allows a single finite syntactic declaration to define an infinite family of distributions.

We define the stochastic state space $\Omega_s$ as the Cartesian product of the domains for all these independent random variables, across all stochastic predicates $r \in \mathcal{R}_s$ and all their ground indices $\mathbf{i} \in \mathcal{I}_r$. The set of ground indices (e.g., time steps $t \in \mathbb{N}$) can be countably infinite, which generally makes $\Omega_s$ an infinite-dimensional product space. The Kolmogorov Extension Theorem ensures that this infinite product of standard Borel spaces, equipped with independent probability measures, forms a valid measurable space $(\Omega_s, \mathcal{F}_{\Omega_s})$ with a well-defined joint probability measure $\mu_s$.

A specific continuous state $\sigma \in \Omega_s$ rigidly dictates the outcome of every stochastic fact in the environment. It maps every grounded index tuple $\mathbf{i}$ of every stochastic predicate $r$ to a specific, constant element $x \in D$. 

During logical evaluation, when the engine encounters a stochastic predicate $r(\mathbf{i}; \mathbf{X})$, the deterministic index $\mathbf{i}$ acts as a lookup key. The state $\sigma$ provides the concrete constant generated for that key, and standard deterministic unification simply binds the empty logical variables $\mathbf{X}$ to that constant. This mechanism completely bypasses the measure-zero paradox while effortlessly supporting unbounded recursion, as each recursive step provides a distinct index (e.g., $T=1, T=2$) mapping to a distinct, independent continuous dimension in $\Omega_s$.

\subsection{Defining the Index Domain ($\mathcal{I}_r$)}
\label{sec:index_domain}
When a logic program evaluates a rule such as $p(\mathbf{X}) \leftarrow r(\mathbf{I}; \mathbf{X})$ where the index $\mathbf{I}$ is unbound by any other predicate in the body, the continuous consequence operator must aggregate over the entire domain of $\mathbf{I}$. To mathematically evaluate this integral, the framework must systematically establish the valid index domain $\mathcal{I}_r$ and its associated base measure. 

MT-PDCL establishes the domain $\mathcal{I}_r$ through two mechanisms within the stochastic measure mapping $\mathcal{M}_s$:

\textbf{1. Implicit Derivation via Parameter Support:} 
If the stochastic declaration maps the indices directly to the parameters of a standard probability distribution, $\mathcal{I}_r$ is implicitly derived as the maximum valid mathematical support of those parameters. For example, given the declaration:
$$ r(M, S; X) \sim \text{Normal}(\mu = M, \sigma = S) $$
The standard deviation of a Normal distribution must be positive, allowing the framework to implicitly derive the valid index domain as the continuous space $\mathcal{I}_r = \mathbb{R} \times (0, \infty)$. The operator aggregates over this space using the standard Lebesgue measure ($dM \, dS$).

\textbf{2. Explicit Definition via Bounded Fields:}
When indices represent physical dimensions (such as spatial coordinates or time intervals) rather than statistical parameters, integrating over an infinite implicit domain is physically invalid. MT-PDCL allows the explicit definition of bounded index spaces directly in the declaration syntax. To maintain syntactic uniformity, these geometric bounds are defined using conjunctions of the language's built-in base predicates ($\mathcal{O}$), effectively defining a continuous random field:
$$ r(\mathbf{I}; \mathbf{X}) \sim \mathcal{D}(\theta(\mathbf{I})) \quad \textbf{over} \quad \Omega_I(\mathbf{I}) $$
For example, a sensor array distributed across a 10-kilometer road can define its boundary using standard arithmetic inequalities:
$$ \text{sensor}(I; X) \sim \text{Normal}(\mu = 0, \sigma = I \cdot 0.1) \quad \textbf{over} \quad I \ge 0 \land I \le 10 $$
By explicitly anchoring the domain with these base predicates, an unbound query to $\text{sensor}(I; X)$ provides the consequence operator with the precise finite limits required to compute the continuous integrals. The specific computational methods needed to calculate these bounded integrals are deferred to practical implementations of the framework.

\subsection{Independent Causal Events}
\label{sec:causal_events}
The second source of uncertainty arises from the probabilistic rules themselves. We model every probability parameter $p_j$ attached to a rule in $\Phi$ as an independent causal event $E_j$. 

This event represents the physical or logical mechanism that allows the rule to successfully fire. It is statistically independent of all other clauses and all continuous variables in $\Omega_s$, and it succeeds with an exact probability $P(E_j) = p_j$.

Let $\Omega_E$ be the set of all possible outcomes for these independent causal events. A specific causal choice $\epsilon \in \Omega_E$ assigns a strict True or False value to every event $E_j$. Mutual independence among these causal events allows the total probability measure $\mu_E$ over the choice space $\Omega_E$ to be defined simply as the standard product measure of the individual event probabilities.

By defining every causal event $E_j$ as independent, MT-PDCL replaces complex global constraint resolution with a generative approach. If a user needs to model correlated rule outcomes, they must explicitly define a shared continuous variable in $\mathcal{M}_s$ and place it in the body of both rules. This independence assumption guarantees that replacing intractable continuous inclusion-exclusion calculations with smooth algebraic operators (such as the continuous Noisy-OR) remains mathematically sound.

\subsection{Generative Rules and Definite Clauses}
\label{sec:generative_rules}
With the state spaces established, we define how rules generate logical truths. Let a probabilistic rule $R_j$ be defined syntactically as $p_j :: H \leftarrow B$.

Under Continuous Distribution Semantics, this syntax is evaluated as a deterministic definite clause combined with the hidden, independent causal event $E_j$:
$$ H \leftarrow B \land E_j $$

This generative definition ensures that the probabilistic uncertainty is isolated entirely to the independent event $E_j$. If the rule body $B$ evaluates to True under a specific continuous state $\sigma \in \Omega_s$, the rule actively generates the head $H$ if and only if $E_j$ evaluates to True. Standard deterministic clauses ($1.0 :: H \leftarrow B$) naturally follow this exact mechanism, generating the head unconditionally whenever the body holds true.

\subsection{The Space of Possible Worlds}
\label{sec:possible_worlds}
A complete MT-PDCL knowledge base is formally defined as the tuple $\mathcal{K} = \langle \mathcal{M}_s, \Phi \rangle$.

The total universe of uncertainty in MT-PDCL is defined by the independent combination of the continuous states and the discrete causal events. We define the space of all possible worlds as the Cartesian product of these two measurable spaces:
$$ \Omega = \Omega_s \times \Omega_E $$

A specific possible world $\omega \in \Omega$ is defined by a pair $\omega = (\sigma, \epsilon)$, where:
\begin{itemize}
    \item $\sigma \in \Omega_s$ rigidly resolves every stochastic predicate queried by the program into a specific continuous constant.
    \item $\epsilon \in \Omega_E$ rigidly determines the True/False outcome of every probabilistic causal event attached to the rules in $\Phi$.
\end{itemize}

Mathematical independence between the continuous states and discrete causal events yields a total probability measure over the space of possible worlds that is the product measure:
$$ \mu_{\Omega} = \mu_s \otimes \mu_E $$

\subsection{Declarative Entailment and Query Answers}
\label{sec:declarative_entailment}
Selecting a specific possible world $\omega = (\sigma, \epsilon)$ eliminates all probabilistic uncertainty. Inside $\omega$, every stochastic predicate yields a fixed constant provided by $\sigma$, and every rule either definitely fires or is completely removed based on $\epsilon$. What remains is a classical, deterministic Definite Clause program, denoted $L_\omega$. 

Standard model theory applies since $L_\omega$ contains no disjunction in the head and no negation. The generative rules act as strict logical constraints, generating one unique Least Model for that specific world. This Least Model is the specific interpretation $\mathcal{I}_\omega$, as defined in Section \ref{sec:inter_pw}. Therefore, a possible world $\omega$ and its induced interpretation $\mathcal{I}_\omega$ are intrinsically linked.

\textbf{Ground Queries:} \\
If a query $Q$ contains no free variables (a ground query), it is either true or false in this specific world. Formally, this is written as $\omega \models Q$, which mechanically means the query is satisfied by the interpretation corresponding to the least model ($\mathcal{I}_\omega \models Q$). We define its exact inferred probability, denoted $P_{\mathcal{K}}(Q)$, as the Lebesgue integral of the logical indicator function over the infinite space of possible worlds:
$$ P_{\mathcal{K}}(Q) = \int_{\Omega} \mathbb{I}(\omega \models Q) \, d\mu_{\Omega}(\omega) $$
where $\mathbb{I}(\omega \models Q) = 1$ if $Q$ is logically entailed by the least model of $L_\omega$, and $0$ otherwise.

\textbf{Non-Ground Queries and Continuous Answers:} \\
In standard definite clause logic, an exact answer to a non-ground query $Q$ (with free variables $\mathbf{V}$) is a discrete variable substitution $\theta$ mapping the variables to specific constants. However, in MT-PDCL, logical variables map natively to continuous domains. In continuous probability theory, the probability of a continuous variable taking a specific point value is zero, making point-substitution evaluations generally invalid. Therefore, an answer in MT-PDCL cannot be reduced to a discrete substitution. Instead, an answer defines a measurable geometric region alongside its exact declarative probability.

We define a \textbf{probabilistic answer} to a non-ground query $Q$ as a tuple $\langle \Omega_A, P_{\mathcal{K}}(Q, \Omega_A) \rangle$:
\begin{itemize}
    \item $\Omega_A \subseteq D^n$ is the \textbf{answer space}, a valid, measurable subset of the continuous universal domain representing a continuous family of bindings for the free variables $\mathbf{V}$. Syntactically, this space is defined by an accumulated conjunction of built-in base predicates (e.g., $V_1 \ge 0 \land V_1 \le 10$).
    \item $P_{\mathcal{K}}(Q, \Omega_A) \in [0, 1]$ is the \textbf{answer probability}. It is the exact declarative probability that the query holds true for the continuous variable bindings bounded by $\Omega_A$ under the knowledge base $\mathcal{K}$.
\end{itemize}

Formally, the answer probability $P_{\mathcal{K}}(Q, \Omega_A)$ for a specific continuous geometric region $\Omega_A$ is defined by integrating the query's indicator function across the space of possible worlds $\Omega$:
$$ P_{\mathcal{K}}(Q, \Omega_A) = \int_{\Omega} \mathbb{I}(\omega \models Q \land \mathbf{V} \in \Omega_A) \, d\mu_{\Omega}(\omega) $$

If the query $Q$ is ground (i.e., it contains no free variables), the tuple of variables $\mathbf{V}$ is the empty tuple. The corresponding answer space $\Omega_A$ reduces to the zero-dimensional space $D^0$. Mathematically, $D^0$ is a singleton set containing only the empty tuple. This unconditionally means that $\mathbf{V} \in \Omega_A$, reducing the probabilistic answer $\langle \Omega_A, P_{\mathcal{K}}(Q, \Omega_A) \rangle$ directly to the scalar probability $P_{\mathcal{K}}(Q)$.

This formulation ensures that the semantics of querying natively respect the underlying continuous measure theory, seamlessly generalizing standard discrete entailment.

\vspace{0.5cm}
\textbf{Theoretical Remark: Bypassing Herbrand Interpretations}

In classical logic programming and discrete probabilistic logics, semantics are traditionally defined using Herbrand interpretations. In a Herbrand model, the domain of discourse is restricted to the finite or countably infinite set of syntactic constants present in the language. 

MT-PDCL bypasses Herbrand interpretations entirely. Continuous domains (such as $\mathbb{R}$) are uncountably infinite and cannot be reduced to a countable syntax of discrete symbols. A stochastic predicate generating continuous values represents an uncountably infinite number of states; exhaustive syntactic propositionalization is mathematically impossible. By anchoring the semantics in general interpretations over a Borel space and defining entailment via exact Lebesgue integration, the framework mathematically bypasses the combinatorial bottleneck of discrete grounding.

\section{Illustrative Examples of Declarative Entailment}\label{sec:examples_declarative}

\subsection{Example: Continuous Spatial Detection}
To demonstrate the semantics of MT-PDCL and illustrate how it replaces discrete summation with exact integration, consider a spatial detection scenario. A security system monitors a specific stretch of a 10km road. The practical objective is to compute the exact scalar probability that a random target event will successfully trigger an alarm, allowing operators to mathematically quantify the system's detection reliability across a continuous physical space.

\subsubsection{Framework Setup}
\begin{itemize}
    \item \textbf{Language $\mathcal{L}$:} A single logical variable $X \in \mathcal{V}$. The predicates are $\text{zone}/1$, $\text{alarm}/0$, and the stochastic predicate $\text{event\_loc}/1 \in \mathcal{R}_s$.
    \item \textbf{Domain $D$:} The continuous interval $[0, 10] \subset \mathbb{R}$, representing the road in kilometers, equipped with the standard Borel $\sigma$-algebra.
    \item \textbf{Stochastic Measure Space $\Omega_s$:} The environment generates one continuous random variable for the event location. The space $\Omega_s$ is $[0, 10]$, and its probability measure $\mu_s$ is defined by the uniform density: $d\mu_s(x) = 0.1 \, dx$.
\end{itemize}

\subsubsection{The Knowledge Base \protect\ensuremath{\mathcal{K}}}
The knowledge base integrates the continuous measure declarations natively alongside the logical clauses:
\begin{enumerate}
    \item $\text{event\_loc}(; X) \sim \text{Uniform}(0, 10)$. \\
    \textit{(Measure declaration: the environment generates a single event location, distributed uniformly.)}
    
    \item $1.0 :: \text{zone}(X) \leftarrow X \in [4, 6]$. \\
    \textit{(Deterministic rule: the sensor zone covers the interval $[4, 6]$.)}
    
    \item $0.8 :: \text{alarm} \leftarrow \text{event\_loc}(; X) \land \text{zone}(X)$. \\
    \textit{(Probabilistic rule: if the event occurs inside the sensor zone, the independent causal event $E_3$ succeeds, triggering the alarm with an 80\% probability.)}
\end{enumerate}

\subsubsection{Generative Evaluation}
To compute the inferred probability of the ground fact \text{alarm}, we first evaluate the deterministic least model of the logic program for a specific continuous state $\sigma \in \Omega_s$. The state resolves the stochastic predicate to a concrete constant $x \in [0, 10]$.

First, we evaluate the logical body of the probabilistic rule: $B = \text{event\_loc}(; x) \land \text{zone}(x)$.
\begin{itemize}
    \item \textbf{Inside the zone ($x \in [4, 6]$):} Both predicates are logically true. With the body evaluating to true, the rule actively generates the \text{alarm} fact upon the success of its associated causal event $E_3$. Since $P(E_3) = 0.8$, the exact conditional probability is:
    $$ P_{\mathcal{K}}(\text{alarm} \mid x) = 0.8 $$

    \item \textbf{Outside the zone ($x \notin [4, 6]$):} The predicate $\text{zone}(x)$ is false, meaning the rule body fails. Under Distribution Semantics, the causal event cannot bridge a false premise. The rule fails to generate the head, yielding an exact probability of zero:
    $$ P_{\mathcal{K}}(\text{alarm} \mid x) = 0 $$
\end{itemize}

\subsubsection{Marginalization over the Measure Space}
We calculate the final, declarative probability of the alarm triggering by integrating the conditional probability over the continuous state space $\Omega_s$, using the environment's measure $\mu_s$:
$$ P_{\mathcal{K}}(\text{alarm}) = \int_{\Omega_s} P_{\mathcal{K}}(\text{alarm} \mid x) \, d\mu_s(x) $$

Substituting our piecewise conditional probabilities and the uniform density $d\mu_s(x) = 0.1 \, dx$:
$$ P_{\mathcal{K}}(\text{alarm}) = \int_{0}^{4} 0 \cdot (0.1) \, dx + \int_{4}^{6} 0.8 \cdot (0.1) \, dx + \int_{6}^{10} 0 \cdot (0.1) \, dx $$
$$ P_{\mathcal{K}}(\text{alarm}) = 0.8 \cdot 0.1 \cdot (6 - 4) = 0.8 \cdot 0.2 = 0.16 $$

This geometric Lebesgue integration yields the exact declarative probability: $P_{\mathcal{K}}(\text{alarm}) = 0.16$. Practically, the framework directly outputs this scalar \texttt{0.16} in response to the query \texttt{?- alarm.}, informing the operator that despite the target being guaranteed to exist on the road, the physical constraints of the sensor zone yield only a 16\% chance of detection.

\subsection{Example: Multi-Sensor Fusion and Discrete Categorization}
To highlight a different aspect of MT-PDCL, specifically how deterministic continuous facts map to discrete, categorical conclusions without integration, we present a scenario based on multi-sensor fusion. The practical objective is to evaluate continuous sensor readings to instantly classify the system's state and issue either a `critical' or `moderate' warning, efficiently pruning impossible logical branches without unnecessary integration.

\subsubsection{Framework Setup}
Consider an industrial monitoring system that reads exact temperature and pressure values to issue these specific warnings.
\begin{itemize}
    \item \textbf{Language $\mathcal{L}$:} Two logical variables $X, Y \in \mathcal{V}$. The predicates are $\text{temp}/1$, $\text{pressure}/1$, and $\text{warning}/1$. The warning predicate takes discrete constants (e.g., \texttt{critical}, \texttt{moderate}).

    \item \textbf{Domain $D$:} A mixed domain containing the continuous set of real numbers $\mathbb{R}$ (for sensor readings) and a finite set of strings (for warning types), equipped with the standard Borel $\sigma$-algebra.
\end{itemize}
\subsubsection{The Knowledge Base \protect\ensuremath{\mathcal{K}}}
Because our observations in this scenario are exact numerical readings rather than stochastic estimates, there are no continuous measure declarations. The knowledge base consists of two deterministic sensor readings and two probabilistic rules:
\begin{enumerate}
    \item $1.0 :: \text{temp}(120.5)$.
    \item $1.0 :: \text{pressure}(4.2)$. \\
    \textit{(Deterministic ground facts: the sensors report an exact continuous temperature of 120.5 and pressure of 4.2).}
    
    \item $0.95 :: \text{warning}(\text{critical}) \leftarrow \text{temp}(X) \land \text{pressure}(Y) \land X > 100.0 \land Y > 4.0$. \\
    \textit{(Probabilistic rule: if both thresholds are breached, issue a critical warning with 95\% probability).}
    
    \item $0.80 :: \text{warning}(\text{moderate}) \leftarrow \text{temp}(X) \land \text{pressure}(Y) \land X \le 100.0 \land Y > 3.0$. \\
    \textit{(Probabilistic rule: if only pressure is high, issue a moderate warning with 80\% probability).}
\end{enumerate}

\subsubsection{Deterministic Evaluation of Causal Events}
Because the environment contains no continuous random variables, the engine evaluates the logical derivations directly against the provided constants.

\textbf{Evaluating Rule 3 (Critical Warning):} 
Through standard unification, the rule's body becomes $B = \text{temp}(120.5) \land \text{pressure}(4.2) \land 120.5 > 100.0 \land 4.2 > 4.0$. Because the sensor readings explicitly exist in $\Phi$ and the mathematical inequalities hold, this body is true. Consequently, the generation of the head depends entirely on the rule's independent causal event $E_3$:
$$ P_{\mathcal{K}}(\text{warning}(\text{critical})) = P(E_3) = 0.95 $$

\textbf{Evaluating Rule 4 (Moderate Warning):} 
For the exact same assignment, the moderate warning body evaluates $120.5 \le 100.0$. Because this mathematical constraint is logically false, the rule body fails. Under generative semantics, a false body means the causal event is irrelevant, and the probability of generating this specific head is zero:
$$ P_{\mathcal{K}}(\text{warning}(\text{moderate})) = 0 $$

By evaluating the queries for both warning states, the framework practically outputs \texttt{0.95} for the critical warning and \texttt{0} for the moderate warning. This resolves the objective, demonstrating how MT-PDCL leverages exact deterministic observations to instantly resolve categorical decisions.

\subsubsection{Contrasting Prior Uncertainty and Backward Compatibility}
Comparing this multi-sensor scenario to the previous spatial detection example reveals a fundamental flexibility of MT-PDCL.

In the first example, the location of the event was unobserved. The framework relied on the stochastic measure mapping to define an integration space, producing an expected probability based on continuous uncertainty.

In this second example, the knowledge base contains exact, deterministic observations. No integration over a prior distribution is required, because the exact observations collapse the evaluation to a definitive set of constants. 

Crucially, this exact scenario can be modeled and executed natively in standard discrete frameworks such as ProbLog or PRISM. Because real numbers acting as fixed deterministic constants do not invoke measure-zero problems, standard discrete solvers can process the arithmetic inequalities without modification. This demonstrates that when the stochastic measure mapping $\mathcal{M}_s$ is empty, MT-PDCL reduces cleanly to classical Distribution Semantics, ensuring strict backward compatibility with existing discrete probabilistic logic programs.

\subsection{Example: Continuous Stochastic Optimization and Querying}
In classical logic programming, constraint satisfaction problems (CSPs) are solved by querying the knowledge base for a set of discrete variable bindings that satisfy a logical goal. MT-PDCL extends this paradigm to continuous stochastic environments. An intelligent agent seeks to find the deterministic parameters that maximize the exact inferred probability of a successful outcome, subject to continuous environmental uncertainty and hard constraints. Practically, the agent requires two distinct outputs: the exact physical energy allocation that maximizes mission success, and the geometric boundaries of valid allocations alongside their baseline success probabilities if queried dynamically.

\subsubsection{Framework Setup}
An agent has a strict total energy budget of 10 units. It must allocate $X$ units to its navigation system and $Y$ units to its communication system. The environment presents continuous, random resistance to both systems. 

To model a realistic asymmetry, the physics of the two systems differ. Signal strength scales linearly with the communication energy $Y$. However, because physical movement (e.g., overcoming aerodynamic drag) requires energy that scales quadratically with velocity, the effective operational yield of the navigation energy $X$ scales only with its square root ($\sqrt{X}$).

\begin{itemize}
    \item \textbf{Language $\mathcal{L}$:} Logical variables $X, Y, R_{nav}, N_{com} \in \mathcal{V}$. The predicates are $\text{nav\_ok}/1$, $\text{com\_ok}/1$, and $\text{mission}/2$. The stochastic predicates are $\text{noise\_nav}/1$ and $\text{noise\_com}/1$.
    \item \textbf{Domain $D$:} The continuous set of real numbers $\mathbb{R}$.
    \item \textbf{Stochastic Measure Space $\Omega_s$:} The environment generates two independent continuous variables. The state space $\Omega_s$ is $\mathbb{R}^2$, mapping to specific environmental outcomes $(r, n)$.
\end{itemize}

\subsubsection{The Knowledge Base \protect\ensuremath{\mathcal{K}}}
The knowledge base explicitly defines the environmental noise distributions and the system rules:
\begin{enumerate}
    \item $\text{noise\_nav}(; R_{nav}) \sim \text{Uniform}(0, 4)$.
    \item $\text{noise\_com}(; N_{com}) \sim \text{Uniform}(0, 10)$. \\
    \textit{(Measure declarations: establishing the independent continuous bounds for terrain resistance and atmospheric noise. Notice the asymmetric bounds.)}
    
    \item $1.0 :: \text{nav\_ok}(X) \leftarrow \text{noise\_nav}(; R_{nav}) \land \sqrt{X} \ge R_{nav}$. \\
    \textit{(The navigation subsystem succeeds if the square root of the allocated energy $X$ equals or exceeds the continuous terrain resistance.)}
    
    \item $0.9 :: \text{com\_ok}(Y) \leftarrow \text{noise\_com}(; N_{com}) \land Y \ge N_{com}$. \\
    \textit{(The communication subsystem succeeds with a 90\% causal probability if the allocated energy $Y$ overcomes the atmospheric noise.)}
    
    \item $1.0 :: \text{mission}(X, Y) \leftarrow \text{nav\_ok}(X) \land \text{com\_ok}(Y) \land X \ge 0 \land Y \ge 0 \land X + Y \le 10$. \\
    \textit{(The overall mission succeeds if both subsystems succeed, provided the energy allocations are non-negative and respect the maximum budget constraint.)}
\end{enumerate}

\subsubsection{Evaluating the Target Function}
To find the optimal allocation, the agent queries the knowledge base for the probability of $\text{mission}(x, y)$ as a continuous function of its specific choices $x$ and $y$. We calculate the entailed probability $P_{\mathcal{K}}(\text{mission}(x, y))$ by evaluating the exact measure of the logically true worlds across $\Omega_s$.

Because the two noise variables are independent, the joint measure $\mu_s$ over $\Omega_s$ is defined by the density function $d\mu_s(r, n) = \frac{1}{4} \cdot \frac{1}{10} \, dr \, dn = \frac{1}{40} \, dr \, dn$ on the support $[0, 4] \times [0, 10]$.

For a specific allocation $(x, y)$ that satisfies the budget constraint $x + y \le 10$, we integrate the generative probabilities over the continuous space where the logical conditions hold ($r \le \sqrt{x}$ and $n \le y$):
$$ P_{\mathcal{K}}(\text{nav\_ok}(x)) = \int_{0}^{\sqrt{x}} 1.0 \cdot \frac{1}{4} \, dr = \frac{\sqrt{x}}{4} $$
$$ P_{\mathcal{K}}(\text{com\_ok}(y)) = \int_{0}^{y} 0.9 \cdot \frac{1}{10} \, dn = \frac{0.9y}{10} $$

Because the subsystems depend on independent variables in the measure space and independent causal events, their joint success probability is the product of their marginals. The deterministic rule for the mission evaluates this conjunction:
$$ P_{\mathcal{K}}(\text{mission}(x, y)) = \frac{\sqrt{x}}{4} \cdot \frac{0.9y}{10} = \frac{0.9 y \sqrt{x}}{40} $$

\subsubsection{Optimization as an Inference Task}
In standard logic programming extensions, optimization is often executed internally using meta-predicates that orchestrate branch-and-bound searches across discrete bindings. However, in MT-PDCL, the probability of a derived fact is a semantic property of the model computed via integration over a measure space, not a syntactic variable bound during rule evaluation. Therefore, optimization is treated as a meta-level inference task evaluated \textit{over} the knowledge base.

The logic program natively defines the continuous objective function representing the declarative probability: $f(x, y) = P_{\mathcal{K}}(\text{mission}(x, y)) = \frac{0.9 y \sqrt{x}}{40}$. The agent queries the solver to maximize this probability subject to the strict logical constraint.

Formally, the agent issues an optimization query:
$$ \arg\max_{x, y} P_{\mathcal{K}}(\text{mission}(x, y)) \quad \text{subject to} \quad x + y \le 10 $$

By substituting $y = 10 - x$ into the objective function, we obtain $f(x) = \frac{0.9}{40}(10\sqrt{x} - x\sqrt{x})$. Setting the derivative to zero yields a non-trivial optimum:
$$ f'(x) = \frac{0.9}{40}\left(\frac{5}{\sqrt{x}} - \frac{3\sqrt{x}}{2}\right) = 0 \implies 10 = 3x \implies x = \frac{10}{3} $$

The exact maximum is reached by allocating $x = \frac{10}{3} \approx 3.33$ units to navigation and $y = \frac{20}{3} \approx 6.67$ units to communication. Because navigation suffers from diminishing returns ($\sqrt{x}$), the optimal strategy correctly allocates double the resources to the linear communication subsystem.

Substituting this optimal strategy back into the declarative engine establishes the exact maximum achievable probability for the mission:
$$ P_{\mathcal{K}}(\text{mission}(10/3, 20/3)) = \frac{0.9 \cdot (20/3) \cdot \sqrt{10/3}}{40} \approx 0.2739 $$

Practically, the framework returns the optimal allocation ($3.33$ units to navigation, $6.67$ to communication) and its maximized 27.3\% success probability. This formal separation ensures the logic remains declarative. The rules define the probability measure space, while the inference engine handles the continuous optimization over that space.

\subsubsection{Executing Ground and Non-Ground Queries}
While optimization searches for a specific maximum, the framework also supports standard logical queries to evaluate specific fixed scenarios or explore continuous parameter spaces.

\textbf{Ground Query:}
An agent can query the exact probability of success for a specific, predetermined energy allocation. Syntactically, this is written as a headless clause with no free variables:
$$ \leftarrow \text{mission}(4.0, 6.0) $$

Because the query is completely ground, the tuple of variables is empty, and the answer space $\Omega_A$ reduces to the zero-dimensional space $D^0$. The declarative engine simply substitutes the constants into the continuous objective function derived from the knowledge base:
$$ P_{\mathcal{K}}(\text{mission}(4.0, 6.0)) = \frac{0.9 \cdot 6.0 \cdot \sqrt{4.0}}{40} = \frac{10.8}{40} = 0.27 $$
The framework returns the exact scalar probability $0.27$.

\textbf{Non-Ground Query:}
An agent can also submit a query containing free variables to explore a valid operational region. For example, the agent may ask for the continuous conditions and the exact probability that \textit{at least one} allocation will succeed if it commits at least 3.0 units of energy to navigation:
$$ \leftarrow \text{mission}(X, Y) \land X \ge 3.0 $$

Because the query contains free variables, the framework returns a probabilistic answer tuple $\langle \Omega_A, \allowbreak P_{\mathcal{K}}(Q, \Omega_A) \rangle$.

First, the engine accumulates the explicit base constraints defined in the rule body and the query to define the geometric answer space $\Omega_A$ for the free variables:
$$ \Omega_A \equiv (X \ge 3.0) \land (Y \ge 0) \land (X + Y \le 10) $$

Next, it calculates the declarative probability $P_{\mathcal{K}}(Q, \Omega_A)$. As defined in Section \ref{sec:declarative_entailment}, the engine does not integrate over the choice variables $X$ and $Y$. Instead, it integrates over the space of possible worlds $\Omega_s$ (the continuous environmental noise). 

In a specific world where the terrain resistance is $r$ and atmospheric noise is $n$, the query holds true if there is at least one valid allocation $(X, Y) \in \Omega_A$ that satisfies the deterministic constraints $\sqrt{X} \ge r$ and $Y \ge n$. This combined region is only satisfiable if $n \le 10 - \max(3, r^2)$.

The engine computes the declarative probability by integrating the joint probability measure $\mu_s$ over the subset of the noise space where this condition holds true, multiplied by the $0.9$ independent causal probability of the communication rule:
$$ P_{\mathcal{K}}(Q, \Omega_A) = 0.9 \int_{0}^{\sqrt{10}} \int_{0}^{10 - \max(3, r^2)} \frac{1}{40} \, dn \, dr $$

Evaluating this exact continuous integral yields the true declarative probability:
$$ P_{\mathcal{K}}(Q, \Omega_A) = \frac{0.9}{40} \left( \frac{20\sqrt{10}}{3} - 2\sqrt{3} \right) \approx 0.396 $$

The framework successfully returns the geometric limits of the valid allocation space alongside its existential probability: $\langle \Omega_A, 0.396 \rangle$. Practically, this tuple provides the agent with the precise geometric region it must operate within ($\Omega_A$) and guarantees a 39.6\% baseline probability of success if it commits to any allocation inside those boundaries.

\vspace{0.5cm}
\textbf{Theoretical Remark: Query Resolution vs. Standard Consequence Operators}

In standard Definite Clause Logic (DCL), answering a query is a direct lookup against the least fixed point of the consequence operator. The inference engine simply unifies the query against the static set of derived ground conclusions. 

This example demonstrates that in MT-PDCL, resolving a non-ground query requires a strict separation of concerns. Because the environment is stochastic and continuous, every possible world $\omega$ generates its own independent fixed point $\mathcal{I}_\omega$. Evaluating an existential query mathematically integrates the probability measure of all worlds where the query's answer space $\Omega_A$ successfully intersects with the world's specific fixed point. Therefore, the framework requires two distinct mechanisms: a continuous consequence operator to formally derive the probability of specific ground instances (detailed in Section \ref{sec:operational_semantics}), and a separate meta-level integration step to resolve free variables across the space of models.

\subsection{Example: Parameterized Distributions and Continuous Marginalization}
To illustrate the semantics of index variables ($\mathbf{I}$) and how deterministic logical variables natively map to the parameters of dynamic probability distributions, we model a mechanical component that fails at a random time, producing a drifting thermal signature. The practical objective is to compute the exact expected probability of a thermal alarm triggering, rigorously accounting for both the unpredictable time of failure and the continuous thermal drift that occurs up to that exact moment.

\subsubsection{Framework Setup}
The component fails at a random continuous time $T$. Upon failure, a thermal sensor reads its temperature $X$. The component grows hotter the longer it runs, meaning the mean of the temperature reading drifts directly as a function of the failure time $T$. 
\begin{itemize}
    \item \textbf{Language $\mathcal{L}$:} Logical variables $T, X \in \mathcal{V}$. The predicates are $\text{alarm}/0$ and two stochastic predicates: $\text{failure\_time}/1$ and $\text{heat\_signature}/2$.
    \item \textbf{Domain $D$:} The continuous set of real numbers $\mathbb{R}$.
\end{itemize}

\subsubsection{The Knowledge Base \protect\ensuremath{\mathcal{K}}}
The knowledge base demonstrates how the deterministic output of one stochastic predicate acts as the binding index for another, naturally parameterizing its distribution:
\begin{enumerate}
    \item $\text{failure\_time}(; T) \sim \text{Uniform}(0, 10)$. \\
    \textit{(The dynamic variable: failure occurs randomly between 0 and 10 hours.)}

    \item $\text{heat\_signature}(T; X) \sim \text{Normal}(\mu = 5.0 \cdot T, \sigma = 2.0)$. \\
    \textit{(The drifting sensor: a stochastic predicate with a non-empty index $T$. Its mean is deterministically parameterized by the given time $T$.)}
    
    \item $1.0 :: \text{alarm} \leftarrow \text{failure\_time}(; T) \land \text{heat\_signature}(T; X) \land X > 30.0$. \\
    \textit{(The deterministic rule: an alarm triggers if the reading exceeds 30 degrees at the exact time of failure.)}
\end{enumerate}

\subsubsection{Conditional Evaluation}
To compute the probability of the alarm, we evaluate the rule for a specific continuous state $\sigma \in \Omega_s$. The first stochastic predicate provides a concrete failure time $t \in [0, 10]$. This time $t$ uniquely binds the index $T$ in the second predicate, fully resolving its distribution to $\text{Normal}(5.0 \cdot t, 2.0)$. 

The body requires $X > 30.0$. Treating the variables as continuous, we integrate the Normal probability density function (PDF), denoted $\mathcal{N}$, to find the conditional probability of the body being true for that specific time $t$:
$$ P_{\mathcal{K}}(\text{alarm} \mid t) = \int_{30}^{\infty} \mathcal{N}(x; 5.0t, 2.0) \, dx = 1 - \Phi\left(\frac{30 - 5.0t}{2.0}\right) $$
where $\Phi$ is the cumulative distribution function (CDF) of the standard normal distribution.

To find the final declarative entailment, we marginalize out the latent variable $T$. The operator computes the Lebesgue integral over the continuous prior measure of the failure time:
$$ P_{\mathcal{K}}(\text{alarm}) = \int_{0}^{10} \left( 1 - \Phi\left(\frac{30 - 5.0t}{2.0}\right) \right) \cdot 0.1 \, dt $$

Evaluating the query \texttt{?- alarm.} returns the scalar result of this integral. This provides the engineering team with the exact expected alarm probability, successfully resolving the dual uncertainty of time and thermal drift. This example highlights the expressive power of MT-PDCL's separation between deterministic indices and stochastic outputs. By simply passing the logical variable $T$ from the body into the index slot of $\text{heat\_signature}$, the logic program effortlessly defines an infinite family of conditional probability distributions and delegates the marginalization to the continuous consequence operator, completely avoiding discrete propositionalization.

\subsection{Example: Hybrid Measure Spaces and Poisson Processes}
To demonstrate how MT-PDCL seamlessly unifies continuous and discrete variables without breaking the mathematical semantics, we model a hybrid dynamic process. The practical objective is for a reliability engineer to calculate the exact probability of a machine suffering a critical failure (defined as experiencing 3 or more stress faults) across an unknown continuous lifespan.

\subsubsection{Framework Setup}
This scenario requires the universal domain to process both continuous time and discrete event counts simultaneously.
\begin{itemize}
    \item \textbf{Language $\mathcal{L}$:} Logical variables $T, K \in \mathcal{V}$. The predicates are $\text{critical\_failure}/0$ and two stochastic predicates: $\text{run\_time}/1$ and $\text{fault\_count}/2$.
    \item \textbf{Domain $D$:} A disjoint union of the continuous real numbers (for time) and the discrete natural numbers (for fault counts), formally structured as a standard Borel space: $D = \mathbb{R} \cup \mathbb{N}$.
\end{itemize}

\subsubsection{The Knowledge Base \protect\ensuremath{\mathcal{K}}}
The declarations highlight a continuous stochastic variable dynamically acting as the index for a discrete probability mass function:
\begin{enumerate}
    \item $\text{run\_time}(; T) \sim \text{Uniform}(0, 5)$. \\
    \textit{(Continuous variable: The machine runs for a random duration between 0 and 5 hours.)}

    \item $\text{fault\_count}(T; K) \sim \text{Poisson}(\lambda = 2.0 \cdot T)$. \\
    \textit{(Discrete variable parameterized by continuous index: The environment generates a discrete count of faults $K \in \mathbb{N}$. The rate $\lambda$ is rigidly parameterized by the continuous operating time $T$.)}
    
    \item $1.0 :: \text{critical\_failure} \leftarrow \text{run\_time}(; T) \land \text{fault\_count}(T; K) \land K \ge 3$. \\
    \textit{(Deterministic rule: The system suffers a critical failure if it experiences 3 or more faults during its operational run.)}
\end{enumerate}

\subsubsection{Hybrid Evaluation and Marginalization}
To compute the exact declarative probability of a critical failure, the continuous consequence operator must navigate both the continuous integration over time and the discrete summation over the Poisson mass function.

First, we evaluate the logical body for a specific continuous state where the machine runs for an exact duration $t \in [0, 5]$. This constant $t$ binds to the index of the Poisson distribution, fixing its parameter to $\lambda = 2.0t$.

The logical condition requires $K \ge 3$. Using the standard Poisson probability mass function $P(K=k) = \frac{\lambda^k e^{-\lambda}}{k!}$, the operator computes the discrete sum for the exact conditional probability that the body evaluates to true:
$$ P_{\mathcal{K}}(\text{critical\_failure} \mid t) = 1 - \sum_{k=0}^{2} \frac{(2.0t)^k e^{-2.0t}}{k!} = 1 - e^{-2.0t}\left(1 + 2.0t + \frac{(2.0t)^2}{2}\right) $$

To find the final declarative entailment, the framework marginalizes out the latent continuous operating time. The operator computes the Lebesgue integral of this discrete conditional summation over the continuous uniform prior density ($d\mu_T(t) = 0.2 \, dt$) across the geometric boundary $[0, 5]$:
$$ P_{\mathcal{K}}(\text{critical\_failure}) = \int_{0}^{5} \left( 1 - e^{-2t}(1 + 2t + 2t^2) \right) \cdot 0.2 \, dt $$

Applying standard integration by parts to the polynomial terms yields an exact, closed-form analytical solution:
$$ P_{\mathcal{K}}(\text{critical\_failure}) = 0.2 \left[ t + e^{-2t} \left(t + \frac{1}{2}\right) + t e^{-2t} + \left(t^2 e^{-2t} + t e^{-2t} + \frac{1}{2} e^{-2t}\right) \right]_{0}^{5} $$
$$ P_{\mathcal{K}}(\text{critical\_failure}) = 0.7 + 7.3e^{-10} \approx 0.7000003 $$

The framework evaluates \texttt{?- critical\_failure.} and outputs \texttt{0.70}. For the reliability engineer, this directly resolves the objective, confirming a 70\% chance the machine will fail by perfectly balancing the continuous uncertainty of its lifespan against the discrete risk of sequential faults. 

This exact calculation highlights a crucial architectural advantage. In standard logic programming frameworks, modeling a Poisson distribution requires discrete recursive unrolling of every possible integer count, causing an immediate combinatorial explosion, followed by manual numerical bucketing of the continuous parameter $T$. In MT-PDCL, the discrete random variable and its continuous parameter exist natively in the same formal measure space, allowing the consequence operator to resolve the hybrid interaction using exact, stable algebraic integration.

\subsection{Example: Continuous Information Cascades in Stochastic Graphs}
To demonstrate how MT-PDCL evaluates dynamic, time-dependent recursion over relational structures, we model the diffusion of information across a social network. An external event triggers a sequence of posts from a small set of root users within a specific time window. The probability that these initial posts are relevant to the event decays continuously over time. Their followers subsequently retweet the information subject to continuous reaction delays and a time-decaying probability of attention. The objective is to compute the probability that a specific target user receives accurate information about the event before a deadline, handling redundant signals and cyclic echo chambers.

\subsubsection{Framework Setup}
\begin{itemize}
    \item \textbf{Language $\mathcal{L}$:} Logical variables $X, Y \in \mathcal{V}$ representing discrete users; $C, K \in \mathcal{V}$ representing discrete tweet counts and indices; and $T, T_{in}, T_{out}, \Delta, A, R \in \mathcal{V}$ mapping to continuous values. The relations are $\text{root}/1$, $\text{follows}/2$, and $\text{target}/1$ (which explicitly model the static topology of the social graph and the query objective), alongside the dynamic relation $\text{informed}/2$. The stochastic predicates are $\text{tweet\_count}/2$, $\text{tweet\_time}/3$, $\text{relevance}/3$, $\text{reaction\_delay}/4$, and $\text{attention}/4$.
    \item \textbf{Domain $D$:} A standard Borel space constructed from the disjoint union of the continuous real numbers $\mathbb{R}$ (representing time and probabilities) and the countable discrete set containing the finite network users $\mathcal{U} = \{\text{u}_1, \text{u}_2, \dots, \text{u}_n\}$ and the natural numbers $\mathbb{N}$ (for sequence indexing).
\end{itemize}

\subsubsection{The Knowledge Base \protect\ensuremath{\mathcal{K}}}
The program explicitly declares the operational time boundaries as deterministic facts, allowing the geometric integration limits to be dynamically queried by the rules.
\begin{enumerate}
    \item $\text{initial\_time}(X; T) \sim \text{Uniform}(0, 5)$. \\
    \textit{(The environment generates the continuous timestamp for a root user $X$'s very first reaction to the event.)}

    \item $\text{inter\_tweet\_delay}(X, K; \Delta) \sim \text{Exponential}(\lambda = 1.0)$. \\
    \textit{(The continuous time delay between a root user $X$'s previous tweet and their next tweet $K$ in the sequence.)}

    \item $\text{relevance}(X, K; R) \sim \text{Uniform}(0, 1)$. \\
    \textit{(An independent continuous variable used to natively model the decaying probability that tweet $K$ for root user $X$ is relevant to the external event.)}

    \item $\text{reaction\_delay}(X, Y, T_{in}; \Delta) \sim \text{Exponential}(\lambda = 0.5)$. \\
    \textit{(The continuous time delay before intermediary user $Y$ retweets user $X$.)}

    \item $\text{attention}(X, Y, T_{in}; A) \sim \text{Uniform}(0, 1)$. \\
    \textit{(An independent continuous variable modeling the dynamic probability that a user pays attention to an incoming tweet.)}

    \item $1.0 :: \text{event\_window}(10.0)$. \quad $1.0 :: \text{deadline}(12.0)$. \\
    \textit{(Deterministic base facts establishing the strict 10-hour boundary for root events and the 12-hour limit for overall signal propagation.)}

    \item $1.0 :: \text{root}(\text{u}_1)$. \quad $1.0 :: \text{follows}(\text{u}_2, \text{u}_1)$. \quad $1.0 :: \text{target}(\text{u}_n)$. $\dots$ \\ 
    \textit{(Deterministic graph topology establishing the social connections and identifying the specific user for the final query.)}

    \item $1.0 :: \text{root\_sequence}(X, 1, T) \leftarrow \text{root}(X) \land \text{event\_window}(T_{max}) \land \text{initial\_time}(X; T) \land T \le T_{max}$. \\
    \textit{(Sequence Base: The first tweet in the train ($K=1$) occurs at the initial continuous time, bounded by the queried event window.)}

    \item $1.0 :: \text{root\_sequence}(X, K, T) \leftarrow \text{root\_sequence}(X, K_{prev}, T_{prev}) \land K = K_{prev} + 1 \land \text{event\_window}(T_{max}) \land \text{inter\_tweet\_delay}(X, K; \Delta) \land T = T_{prev} + \Delta \land T \le T_{max}$. \\
    \textit{(Sequence Recursion: Subsequent tweets in the train are generated by adding a strictly positive continuous delay. The train naturally halts when $T$ exceeds the queried parameter $T_{max}$.)}

    \item $1.0 :: \text{informed}(X, T) \leftarrow \text{root\_sequence}(X, K, T) \land \text{relevance}(X, K; R) \land R \le \exp(-0.5 \cdot T)$. \\
    \textit{(Root Information: A root user is informed at time $T$ if tweet $K$ is generated and its relevance $R$ satisfies the time-decaying threshold.)}

    \item $0.6 :: \text{informed}(Y, T_{out}) \leftarrow \text{follows}(Y, X) \land \text{informed}(X, T_{in}) \land \text{attention}(X, Y, T_{in}; A) \land A \le \exp(-0.1 \cdot T_{in}) \land \text{reaction\_delay}(X, Y, T_{in}; \Delta) \land T_{out} = T_{in} + \Delta \land \text{deadline}(T_{dead}) \land T_{out} \le T_{dead}$. \\
    \textit{(Recursive propagation: The message propagates subject to continuous thresholds and delays, bounded strictly by the global $T_{dead}$ parameter queried from the knowledge base.)}
\end{enumerate}

\subsubsection{Cyclic Recursion and Continuous Aggregation}
Because MT-PDCL evaluates semantics via a bottom-up consequence operator ($\mathbb{T}_{\mathcal{K}}$), the proof advances forward in time. The operator starts with the initial event facts and iteratively pushes the probability mass through the network, accumulating the continuous delays $\Delta$ at each hop.

Once the consequence operator reaches its least fixed point, the objective 
$$\leftarrow \text{target}(X) \land \text{informed}(X, T)$$ 
acts as a final projection. The framework isolates the target user $X$ from the generated model and marginalizes (integrates) over the accumulated time variable $T$ to output a single scalar: the total probability that the target user was informed.

Because the social network graph contains cyclic paths, naive bottom-up evaluation would trigger infinite loops. Furthermore, a user may receive the same information multiple times. MT-PDCL handles this complexity in two ways. First, the continuous Noisy-OR operator aggregates overlapping derivations via integration over the latent time distributions. Second, the operator computes the limit of the cyclic recursion. Because the time delay $\Delta$ is strictly positive, the accumulated time grows monotonically along any cycle. The physical constraint bounding the propagation ($T_{out} \le T_{dead}$) forces the integration volume to shrink to zero as cycles repeat, halting the recursion and guaranteeing convergence.

\subsection{Theoretical Justification: Bounded Approximation Error}
To formally justify the accuracy of the MT-PDCL consequence operator on highly cyclic, dense networks, we quantify the approximation error introduced by the Noisy-OR assumption. While dense topologies cause the number of overlapping derivation paths to grow rapidly, the monotonic progression of time reduces the probability mass of longer paths. The following theorem establishes that the absolute error gap between the operational Noisy-OR probability and the exact declarative probability is bounded and contracts exponentially.

\begin{theorem}[Bounded Approximation Error for Time-Decaying Cascades]
\label{thm:bounded_error}
Let $G = (\mathcal{V}, \mathcal{E})$ be a directed network graph, and let $S$ be the finite set of valid derivation paths from any root to a target node, bounded by the physical deadline $T_{dead}$. Let $k_{min} \ge 1$ be the minimum discrete path length (number of hops) in $S$. 

Assume the continuous temporal transition across any edge is given by $\Delta \sim \text{Exponential}(\lambda_D)$ with $\lambda_D > 0$. Let the causal propagation probability at hop $j \ge 1$ be strictly bounded by $\exp(-\lambda_A T_j)$, where $T_j$ is the accumulated continuous time. Furthermore, let the initial relevance probability at the root ($j=0$) be bounded by $\exp(-\lambda_R T_0)$, with $\lambda_R > 0$.

If $P_{op}$ is the Noisy-OR operational probability and $P_{ex}$ is the exact maximally-correlated probability, then the absolute approximation error $\epsilon = | P_{op} - P_{ex} |$ is bounded by:
$$ \epsilon \le \frac{1}{2} |S|^2 \mathbb{E}[\exp(-2 \lambda_R T_0)] \left( \frac{\lambda_D}{\lambda_D + \lambda_A} \right)^{2 k_{min}} $$
\end{theorem}

\begin{proof}
By the expansion of the Noisy-OR operator, the absolute error $\epsilon$ introduced by assuming path independence is bounded by the sum of all pairwise joint path probabilities. This simplifies to:
$$ \epsilon \le \frac{1}{2} \left( \sum_{\pi \in S} \mathbb{E}[p(\pi)] \right)^2 $$

For any path $\pi$ of length $k$, the accumulated continuous time at hop $j$ is $T_j = T_0 + \sum_{m=1}^j \Delta_m$. The expected probability mass for the path is constrained by both the initial root relevance decay and the subsequent propagation attention decay at each hop. Because $T_j \ge \Delta_j$, we can bound the time-decay at each individual hop by $\exp(-\lambda_A \Delta_j)$. Since the delays $\Delta_m$ are independent and Exponentially distributed, the expectation of the product evaluates directly via the moment-generating function:
$$ \mathbb{E}[p(\pi)] \le \mathbb{E}[\exp(-\lambda_R T_0)] \prod_{m=1}^k \mathbb{E} \left[ \exp(-\lambda_A \Delta_m) \right] = \mathbb{E}[\exp(-\lambda_R T_0)] \left( \frac{\lambda_D}{\lambda_D + \lambda_A} \right)^k $$

Because $\lambda_D$ and $\lambda_A$ are positive, the ratio $\frac{\lambda_D}{\lambda_D + \lambda_A}$ is strictly less than 1. Therefore, this geometric decay is maximized at the shortest valid path length $k_{min}$. Substituting this upper bound into the pairwise error sum and squaring the result over the $|S|$ valid paths yields the final bounded limit:
$$ \epsilon \le \frac{1}{2} |S|^2 \mathbb{E}[\exp(-2 \lambda_R T_0)] \left( \frac{\lambda_D}{\lambda_D + \lambda_A} \right)^{2 k_{min}} $$
\end{proof}

\textbf{Practical Outcome:}
The framework evaluates the query and outputs the exact scalar probability that the target user receives the information in time. Practically, this allows network analysts to quantify information reachability under dynamic, continuous constraints. 

This example resolves a computational limitation in existing discrete solvers. Modeling this scenario in a traditional probabilistic logic framework requires discretizing the 12-hour window into fixed time steps. Grounding a highly connected, cyclic social graph over discrete time steps causes exponential growth in the underlying boolean circuit size. MT-PDCL avoids this issue by evaluating the temporal boundary as a unified geometric region within the continuous consequence operator.

\section{Operational Semantics and Theoretical Properties}\label{sec:operational_semantics}

\subsection{The Continuous Probabilistic Consequence Operator}
Before formalizing the consequence operator, we must define what constitutes a derivation in a measure-theoretic setting. In standard logic programming, a derivation is a discrete sequence of rule applications that proves a specific conclusion. In MT-PDCL, a \textbf{derivation} is a specific logical path that contributes probability mass to a conclusion. Distinct derivations for the same conclusion arise in two ways: structurally, from completely different rules whose heads conclude the same relation; and continuously, from disjoint subsets of continuous variables evaluating the same rule body. When a knowledge base contains multiple such derivations for the same conclusion, their probabilistic contributions must be aggregated.

In standard definite clause logic, operational semantics are defined by the immediate consequence operator ($T_P$). Given a current interpretation, $T_P$ evaluates all rule bodies simultaneously against that interpretation to derive a new set of facts. To operationalize MT-PDCL, we must generalize this concept to act natively upon the space of continuous probabilities. 

Let $\mathcal{P}$ denote the space of probability functions mapping ground facts to $[0, 1]$. Because the domain $D$ contains continuous variables, the set of possible ground facts is uncountably infinite. Thus, a function $P \in \mathcal{P}$ represents a continuous probability surface rather than a discrete table. We define the generic continuous probabilistic consequence operator $\mathbb{T}: \mathcal{P} \rightarrow \mathcal{P}$. This operator takes an input probability function $P$ from iteration $k$ and produces an updated function $\mathbb{T}(P)$ for iteration $k+1$ by systematically applying the generative rules in $\Phi$.

To construct $\mathbb{T}$, we must define what constitutes a rule application. In discrete logic, the forward inference step relies on Generalized Modus Ponens. In MT-PDCL, we define a \textbf{rule instance} as the application of Generalized Modus Ponens for a clause $R_j \in \Phi$ anchored to a specific continuous variable binding $\mathbf{v} \in \Omega_{V_j}$. Here, $\Omega_{V_j}$ represents the joint measurable domain of all variables (both deterministic indices and stochastic outputs) evaluated in the rule body. While we generally refer to this space as continuous, measure theory seamlessly supports discrete probability generators (e.g., categorical distributions) as a special case by treating them as discrete measures (point masses) over the Borel space. Thus, for continuous variables, a single rule represents an uncountably infinite number of instances, whereas for discrete variables, the measure natively reduces this to a finite or countable set of instances.

Consider a specific rule instance for $R_j$ evaluated at binding $\mathbf{v}$, where $R_j$ is of the form $p_j :: H \leftarrow B$. Under generative semantics, the rule generates the head if and only if the body is true \textit{and} the causal event $E_j$ succeeds. The independence of $E_j$ from the stochastic environment and all prior logical derivations dictates that the probability of this joint occurrence is the product of their individual probabilities. Therefore, if the body $B$ currently evaluates to a probability under the input function $P$ (denoted $P(B \mid \mathbf{v})$) for this specific binding, and the causal event succeeds with probability $p_j$, the isolated, intermediate probabilistic contribution of this specific rule instance to the head $H$ is:
$$ P_{\text{update}}(H \mid R_j, \mathbf{v}) = p_j \cdot P(B \mid \mathbf{v}) $$

For a standard deterministic clause ($1.0 :: H \leftarrow B$), the parameter is $p_j=1.0$, meaning the exact probability of the body is propagated directly to the head:
$$ P_{\text{update}}(H \mid R_j, \mathbf{v}) = 1.0 \cdot P(B \mid \mathbf{v}) $$

If $H$ is the head of only one rule, and only a single valid binding $\mathbf{v}$ satisfies the body, the final updated probability for the head in iteration $k+1$ is this isolated computation. However, logic programs inherently utilize multiple overlapping derivations. A knowledge base frequently contains multiple distinct rules that derive the same relation $H$, and a single rule generates the exact same ground head fact across a continuous domain of variables. 

To ensure $\mathbb{T}$ outputs a single, unified probability for $H$, we must formalize an aggregation mechanism that safely combines the probability masses from multiple successful rule applications.

\subsection{Exact Aggregation via Continuous Unions}
Under discrete Distribution Semantics (such as in ProbLog), if multiple rule instances derive the same head fact $H$, the evaluation uses logical disjunction: $H$ is true if \textit{any} of its derivations hold. Simply adding individual probabilities would overcount, as derivations often overlap by sharing subgoals or variables. Calculating the exact probability therefore requires the inclusion-exclusion principle: $P(A \lor B) = P(A) + P(B) - P(A \land B)$.

In MT-PDCL, we must adapt this disjunctive aggregation to operate natively over our continuous possible worlds $\Omega$. Let $\mathcal{R}_H \subseteq \Phi$ be the finite set of rules whose heads unify with a specific ground instance of a fact $H$. We define the exact consequence operator $\mathbb{T}_{\text{exact}}$. The exact updated probability for $H$ produced by $\mathbb{T}_{\text{exact}}$ is the Lebesgue measure of the infinite geometric union of success regions across all rules and their continuous evaluation domains:
$$ \mathbb{T}_{\text{exact}}(P)(H) = \mu_{\Omega} \left( \bigcup_{R_j \in \mathcal{R}_H} \bigcup_{\mathbf{v} \in \Omega_{V_j}} \Omega_{\text{success}}(R_j, \mathbf{v}) \right) $$

\noindent where $\Omega_{\text{success}}(R_j, \mathbf{v}) \subseteq \Omega$ represents the specific subset of possible worlds where the instance $(R_j, \mathbf{v})$ successfully triggers (i.e., the body is logically true in the world's least model, and the causal event $E_j$ evaluates to True). 

This geometric union is the direct measure-theoretic equivalent of the existential quantifier applied to Generalized Modus Ponens: $H$ is true if there exists \textit{any} rule and \textit{any} valid variable binding that derives it. When the variables are discrete, this continuous formulation mathematically collapses back to the standard countable union used in discrete probabilistic logic.

While this formula defines the exact declarative truth, applying the inclusion-exclusion principle to compute the exact geometric volume of an uncountably infinite union of overlapping continuous sets is computationally intractable. 

\vspace{0.5cm}
\textbf{Remark: Iteration and the Fixed Point.} \\
Because logic programs utilize chained and recursive rules, a single application of the consequence operator rarely captures the complete probability mass. Instead, the operator is applied iteratively, creating a sequence of probability functions $P_0, P_1, P_2, \dots$ starting from a baseline where all facts have zero probability ($P_0 = P_\bot$). The iteration stops when a stable state is reached where evaluating the rules yields no further changes, meaning $\mathbb{T}_{\text{exact}}(P_k) = P_k$. This stable state is called the \textit{least fixed point} (denoted $lfp$), and it represents the final, logically entailed probabilities of all facts. We formally prove the existence of and convergence to this $lfp$ in Section \ref{subsec:monotonicity}, but we introduce the basic mechanism here to evaluate the upcoming example.

\subsection{Example: Exact Continuous Aggregation}\label{subsec:multi-sensor}
To illustrate the mechanics of $\mathbb{T}_{\text{exact}}$, we define a simple knowledge base where a target exists across a 10km domain. Two independent sensors attempt to measure its location, generating overlapping continuous rule instances. 

\textbf{The Knowledge Base $\mathcal{K}$:}
\begin{enumerate}
    \item $\text{target\_loc}(; X) \sim \text{Uniform}(0, 10)$. \\
    \textit{(Measure declaration: generates a target location $X$, distributed uniformly.)}
    
    \item $1.0 :: \text{target}(X) \leftarrow \text{target\_loc}(; X)$. \\
    \textit{(Rule 1: Deterministically places the target at the generated location.)}
    
    \item $0.6 :: \text{track} \leftarrow \text{target}(X) \land X \in [2, 6]$. \\
    \textit{(Rule 2: Sensor A detects the target with 0.6 probability if $X \in [2, 6]$.)}
    
    \item $0.8 :: \text{track} \leftarrow \text{target}(X) \land X \in [5, 9]$. \\
    \textit{(Rule 3: Sensor B detects the target with 0.8 probability if $X \in [5, 9]$.)}
    
    \item $1.0 :: \text{verified} \leftarrow \text{track}$. \\
    \textit{(Rule 4: A derived track is immediately verified.)}
\end{enumerate}

We apply $\mathbb{T}_{\text{exact}}$ iteratively, starting from $P_0$, where all derived facts have zero probability.

\textbf{Step 1: Establishing the target location ($P_1$)} \\
Rule 1 evaluates directly against the mathematical measure $\mathcal{M}_s$. It unconditionally assigns a probability of 1.0 to the target's existence for any specific location $x \in [0, 10]$ provided by the environment, establishing the continuous prior density $d\mu_s(x) = 0.1 \, dx$.

\textbf{Step 2: Taking measurements and aggregating derivations ($P_2$)} \\
Rules 2 and 3 evaluate against $P_1$. Because both sensors measure the exact same target location $x$, their logical derivations overlap. Assuming the sensors' physical causal events operate independently for any specific location $x$, we apply the exact point-wise inclusion-exclusion principle: $P(A \lor B) = P(A) + P(B) - P(A \land B)$.

We partition the continuous evaluation space of the shared variable $X$ to compute the exact aggregated measure for \text{track}:
\begin{itemize}
    \item \textbf{Interval $x \in [2, 5)$:} Only Sensor A successfully measures the target.
    $$ \int_{2}^{5} 0.6 \cdot 0.1 \, dx = 0.6 \cdot 0.3 = 0.18 $$
    
    \item \textbf{Interval $x \in [5, 6]$:} Both sensors successfully derive the head. Because the causal events are independent, the combined probability of success at point $x$ is $0.6 + 0.8 - (0.6 \cdot 0.8) = 0.92$. 
    $$ \int_{5}^{6} 0.92 \cdot 0.1 \, dx = 0.92 \cdot 0.1 = 0.092 $$
    
    \item \textbf{Interval $x \in (6, 9]$:} Only Sensor B successfully measures the target. 
    $$ \int_{6}^{9} 0.8 \cdot 0.1 \, dx = 0.8 \cdot 0.3 = 0.24 $$
\end{itemize}

The uncountably infinite union from the $\mathbb{T}_{\text{exact}}$ definition integrates these point-wise geometric probabilities over the shared target measure:
$$ P_2(\text{track}) = 0.18 + 0.092 + 0.24 = 0.512 $$

\textbf{Step 3 \& 4: Reaching the fixed point} \\
Rule 4 propagates the track probability directly ($1.0 \cdot 0.512 = 0.512$). Evaluating $\Phi$ again yields no new derivations. The operator has reached its fixed point. This computed exact value identically matches the declarative entailment defined by the Lebesgue integral over the possible worlds space in Section \ref{sec:declarative_entailment}:
$$ lfp(\mathbb{T}_{\text{exact}})(\text{verified}) = P_{\mathcal{K}}(\text{verified}) = 0.512 $$

\textbf{Remark on Computational Tractability:} \\
This example demonstrates why exact aggregation is generally intractable. Here, the overlap could be solved analytically because the regions were one-dimensional and uniform. In a general knowledge base with multi-dimensional continuous variables and deep recursion, computing the exact intersection boundaries $\Omega_{\text{success}}(R_1) \cap \Omega_{\text{success}}(R_2)$ to perform exact continuous inclusion-exclusion is practically impossible.

\subsection{Operationalizing Aggregation: Continuous Noisy-OR}
To make continuous aggregation computable, we introduce a structural approximation regarding how multiple derivations interact. A standard approach in probabilistic logic programming assumes that multiple rules supporting the same head fact operate independently. 

In the AI literature, this is known as the \textit{Independent Causal Influence} (ICI) assumption \cite{pearl1988probabilistic, heckerman1996causal}. Mathematically, it assumes the \textit{conditional independence} of rule failures.

Under ICI, rather than calculating the intractable infinite union of successful geometric derivations, we calculate the intersection of their failures. Because rule failures are assumed conditionally independent, this intersection simplifies to the strict product of their failure probabilities. Fact $H$ is considered true if at least one derivation succeeds, computed as $1$ minus the probability that all derivations fail. 

For a single rule $R_j$, the probability that a specific instance evaluated at continuous binding $\mathbf{v}$ fails to generate the head is $1 - P_{\text{update}}(H \mid R_j, \mathbf{v})$. When a rule evaluates over a continuous space $\Omega_{V_j}$, it represents an uncountably infinite family of instances. To calculate their joint failure, we must generalize the discrete product over this continuous domain. 

Taking the logarithm of a product transforms it into a sum: $\ln\left( \prod (1 - p) \right) = \sum \ln(1 - p)$. In the continuous limit, this sum translates into a Lebesgue integral over the continuous evaluation domain. Applying the exponential function recovers the actual probability. 

Thus, the total assumed probability that the entire continuous family of rule instances for $R_j$ fails is computed over its bound domain $\Omega_{V_j}$ with base measure $\mu_V$:
$$ P_{\text{fail}}(R_j) = \exp\left( \int_{\Omega_{V_j}} \ln\Big(1 - P_{\text{update}}(H \mid R_j, \mathbf{v})\Big) \, d\mu_V(\mathbf{v}) \right) $$

To compute the final, unified probability for the head fact $H$ under the ICI assumption, the operational consequence operator $\mathbb{T}_{\mathcal{K}}$ multiplies these continuous failure probabilities across all $m$ rules in $\mathcal{R}_H$, and subtracts the result from 1:
$$ \mathbb{T}_{\mathcal{K}}(P)(H) = 1 - \prod_{R_j \in \mathcal{R}_H} P_{\text{fail}}(R_j) $$

\textbf{Remark: Deterministic Bounds and the Logarithmic Singularity.} \\
When evaluating a fully deterministic derivation where the update probability reaches $1.0$, the failure probability is zero, resulting in a logarithmic singularity: $\ln(0) \to -\infty$. Integrating $-\infty$ over a subset of positive measure causes the integral to diverge to $-\infty$. The outer exponential evaluates to $\exp(-\infty) = 0$, producing a final Noisy-OR success probability of $1.0$. This perfectly preserves logical semantics: if any continuous subset of derivations is guaranteed to succeed, the aggregated conclusion is unconditionally true.

To see how this approximation changes the computation, we revisit the multi-sensor tracking example from Section \ref{subsec:multi-sensor}. In the exact calculation, we had to manually partition the continuous domain into distinct intervals ($[2, 5)$, $[5, 6]$, and $(6, 9]$) to correctly calculate the geometric overlap between the two sensors. 

The continuous Noisy-OR operator avoids tracking these complex intersections entirely. Instead, it computes the joint failure probability for each rule independently across its full spatial bound using the product integral, and then combines the results.

\textbf{Recalculating Step 2: Aggregation via Continuous Noisy-OR} \\
For Rule 2 (Sensor A, measuring over $x \in [2, 6]$):
$$ P_{\text{fail}}(R_2) = \exp\left( \int_{2}^{6} \ln(1 - 0.6) \cdot 0.1 \, dx \right) = \exp(0.4 \cdot \ln(0.4)) \approx 0.693 $$

For Rule 3 (Sensor B, measuring over $x \in [5, 9]$):
$$ P_{\text{fail}}(R_3) = \exp\left( \int_{5}^{9} \ln(1 - 0.8) \cdot 0.1 \, dx \right) = \exp(0.4 \cdot \ln(0.2)) \approx 0.525 $$

The continuous Noisy-OR operator aggregates these independent rule failures to update the head fact:
$$ P_2(\text{track}) = 1 - (0.693 \cdot 0.525) \approx 0.636 $$
Propagating to the fixed point yields the operational approximation: $lfp(\mathbb{T}_{\mathcal{K}})(\text{verified}) = P_{op}(\text{verified}) \approx 0.636$.

\textbf{Theoretical Insight: Divergence of Results.} \\
The continuous Noisy-OR approximation ($0.636$) overestimates the exact measure-theoretic computation ($0.512$). This divergence highlights what the ICI assumption sacrifices. 

The exact computation correctly treats the continuous variable $X$ as a single, physical entity: the target's locations are \textit{mutually exclusive}. A single target can exist at $x_1$ or $x_2$, but never both simultaneously. 

Conversely, the continuous Noisy-OR operator assumes conditional independence across the \textit{entire} integration domain. By using the $\exp \left( \int \ln \right)$ formula, it ignores the physical constraint of mutual exclusivity. Mathematically, it treats every point $x$ in the continuous space as a separate, independent target capable of triggering the rule. Because it evaluates the domain as a collection of independent events rather than a single mutually exclusive variable, it computes a physically incorrect overestimate. 

This establishes a critical theoretical boundary: the continuous Noisy-OR operator is a tractable, differentiable algebraic proxy for exact logical inclusion-exclusion, but it sacrifices accuracy when variables represent mutually exclusive physical states.

\subsection{Monotonicity and Convergence}\label{subsec:monotonicity}
In standard logic programming, the semantics of the consequence operator rely on two fundamental theorems: the Knaster-Tarski theorem, which guarantees the existence of a least fixed point ($lfp$) for monotonic operators; and Kleene's fixed-point theorem, which guarantees that iterative evaluation converges to this $lfp$, provided the operator is continuous \cite{lloyd1987foundations}. We establish a complete lattice over $\mathcal{P}$, the space of continuous probability functions, to apply these properties to MT-PDCL.

Let $P_1, P_2 \in \mathcal{P}$. We define a partial order $\sqsubseteq$ such that $P_1 \sqsubseteq P_2$ if and only if for every ground fact $H$, $P_1(H) \leq P_2(H)$. 

The space $\mathcal{P}$ is bounded below by the zero function $P_\bot$ (where $P_\bot(H) = 0$ for all facts) and bounded above by the unit function $P_\top$ (where $P_\top(H) = 1$ for all facts). 

To form a complete lattice, every subset $S \subseteq \mathcal{P}$ must have a least upper bound (supremum) and a greatest lower bound (infimum) in $\mathcal{P}$. We define these point-wise. For any ground fact $H$:
$$ (\sup S)(H) = \sup_{P \in S} P(H) $$
$$ (\inf S)(H) = \inf_{P \in S} P(H) $$
The real interval $[0, 1]$ is a complete lattice, guaranteeing that the supremum and infimum of any set of values in $[0, 1]$ always exist and remain within $[0, 1]$. The functions $\sup S$ and $\inf S$ therefore exist and belong to $\mathcal{P}$. Thus, the space of functions $\mathcal{P}$ under the point-wise ordering $\sqsubseteq$ forms a complete lattice.

\begin{proposition}[Monotonicity and Convergence of $\mathbb{T}$]
Both the exact operator $\mathbb{T}_{\text{exact}}$ and the continuous Noisy-OR operator $\mathbb{T}_{\mathcal{K}}$ are monotonically increasing over $\mathcal{P}$. Consequently, they possess a least fixed point, and their iterative application starting from the zero function $P_\bot$ converges to it.
\end{proposition}

\begin{proof} 
\textit{Part 1: Monotonicity and Existence.} Let $P_1, P_2 \in \mathcal{P}$ such that $P_1 \sqsubseteq P_2$. By definition, $P_1(H) \leq P_2(H)$ for all facts $H$. MT-PDCL evaluates definite clauses, ensuring the probability of any conjunction of positive literals evaluates monotonically. Thus, for any rule body $B$ evaluated at binding $\mathbf{v}$, we have $P_1(B \mid \mathbf{v}) \leq P_2(B \mid \mathbf{v})$. 

For $\mathbb{T}_{\text{exact}}$, an increase in body probabilities expands the geometric subsets of successful possible worlds. Therefore, the success regions generated under $P_1$ are subsets of those generated under $P_2$. The Lebesgue measure of a superset is strictly non-decreasing, ensuring $\mathbb{T}_{\text{exact}}(P_1) \sqsubseteq \mathbb{T}_{\text{exact}}(P_2)$.

For $\mathbb{T}_{\mathcal{K}}$, the isolated update probability $P_{\text{update}}(H \mid R_j, \mathbf{v}) = p_j \cdot P(B \mid \mathbf{v})$ increases under $P_2$. Consequently, the local failure probability $(1 - P_{\text{update}})$ decreases. The logarithm of a smaller positive fraction yields a more negative value, making the resulting integral more negative, and its exponential closer to 0. Subtracting this smaller exponential from 1 yields a larger aggregated probability for the head. Thus, $\mathbb{T}_{\mathcal{K}}(P_1) \sqsubseteq \mathbb{T}_{\mathcal{K}}(P_2)$. 

Because both operators are monotonic over a complete lattice, the Knaster-Tarski theorem guarantees the existence of a $lfp$.

\textit{Part 2: $\omega$-Continuity and Convergence.} To guarantee that iterative application from $P_\bot$ reaches the $lfp$, the operators must be $\omega$-continuous, meaning they preserve the suprema of ascending chains (i.e., $\mathbb{T}(\sup_{k} P_k) = \sup_{k} \mathbb{T}(P_k)$). The exact operator $\mathbb{T}_{\text{exact}}$ is constructed from standard probability measures, which are inherently continuous from below \cite{billingsley2012probability}. The operational operator $\mathbb{T}_{\mathcal{K}}$ is composed entirely of continuous algebraic operations (integration, logarithms, and exponentials) over a bounded space. Therefore, both operators are $\omega$-continuous. By Kleene's fixed-point theorem, the application of an $\omega$-continuous, monotonic operator starting from the bottom element $P_\bot$ converges to the $lfp$. 
\end{proof}

\subsection{Iteration Bounds and the Contraction Mapping}\label{subsec:computation_bounds}
While Kleene's theorem guarantees that iterative evaluation will eventually converge to the true fixed point, it does not provide a computable bound on the number of iterations required. To estimate the number of computational steps needed for a given precision, we must analyze the operators using the Banach Fixed-Point Theorem (Contraction Mapping Theorem) \cite{banach1922operations}. 

To do this, we treat $\mathcal{P}$ as a metric space rather than just a partial order. We define the distance between two probability functions $P_a, P_b \in \mathcal{P}$ using the supremum norm:
$$ d(P_a, P_b) = \sup_{H} |P_a(H) - P_b(H)| $$
This metric measures the maximum probability difference for any ground fact across the entire continuous domain. 

To apply the Banach theorem, an operator must be a strict contraction, meaning there exists a Lipschitz constant $L < 1$ such that $d(\mathbb{T}(P_a), \mathbb{T}(P_b)) \leq L \cdot d(P_a, P_b)$. 

In MT-PDCL, the contraction property depends entirely on the recursive structure of the logic program. To formally identify recursive rules, we define the program's dependency graph, where nodes are predicates and a directed edge exists from predicate $H$ to predicate $B$ if a rule derives $H$ and contains $B$ in its body. Let $\mathcal{R}_{\text{cycle}} \subseteq \Phi$ be the set of all rules whose head and at least one body literal belong to the same Strongly Connected Component (SCC). We refer to any predicate belonging to such an SCC (i.e., the head of any rule in $\mathcal{R}_{\text{cycle}}$) as a \textit{cyclic predicate}.

To evaluate if the consequence operator $\mathbb{T}$ acts as a contraction mapping, we first formally define a global bound on the recursive derivations.

\textbf{Definition 1 (Recursive Measure $\gamma$).} Let the expected number of successful derivations for a cyclic ground fact be the continuous integral of its causal probabilities. We define $\gamma$ as the supremum of this measure across all possible ground instances of any cyclic predicate:
$$ \gamma = \sup_{\mathbf{h}} \sum_{R_j \in \mathcal{R}_{\text{cycle}}} \int_{\Omega_{V_j}} p_j(\mathbf{v}) \, d\mu_V(\mathbf{v}) $$

\begin{proposition}[Lipschitz Bound and Contraction Regimes]
Under the continuous Independent Causal Influence (ICI) assumption, the consequence operator $\mathbb{T}$ satisfies the Lipschitz condition:
$$ d(\mathbb{T}(P_a), \mathbb{T}(P_b)) \leq \gamma \cdot d(P_a, P_b) $$
Consequently, the operator falls into one of two iteration regimes:
\begin{enumerate}
    \item \textbf{Strict Contraction ($\gamma < 1$):} The operator $\mathbb{T}$ is a strict contraction mapping with a Lipschitz constant $L = \gamma$, ensuring geometric convergence to the exact least fixed point.
    \item \textbf{Expansive Propagation ($\gamma \geq 1$):} The Lipschitz bound is $L \geq 1$. The mapping is not a contraction, preventing the calculation of a strict Banach iteration bound.
\end{enumerate}
\end{proposition}

\begin{proof}
Let $\Delta = d(P_a, P_b) = \sup_{\mathbf{v}} |P_a(\mathbf{v}) - P_b(\mathbf{v})|$. We parameterize a straight-line path between the two functions for $t \in [0, 1]$ such that $P_t = (1 - t)P_a + t P_b$. 

For a specific ground instance of a cyclic predicate, let $F(t)$ be its operational update under the interpolated probability $P_t$. Using the continuous Noisy-OR aggregation, $F(t)$ is explicitly defined as:
$$ F(t) = 1 - E_0 \cdot \exp\left( \sum_{R_j \in \mathcal{R}_{\text{cycle}}} \int_{\Omega_{V_j}} \ln\Big(1 - P_t(\mathbf{v}) \cdot p_j(\mathbf{v})\Big) \, d\mu_V(\mathbf{v}) \right) $$
where $E_0 \in [0, 1]$ represents the constant joint failure probability from any non-cyclic base rules deriving this fact. 

Taking the derivative of $F(t)$ with respect to $t$ using the chain rule, and factoring out the supremum distance $\Delta \geq |P_b(\mathbf{v}) - P_a(\mathbf{v})|$, yields:
$$ |F'(t)| \leq \Delta \cdot (1 - F(t)) \sum_{R_j \in \mathcal{R}_{\text{cycle}}} \int_{\Omega_{V_j}} \frac{p_j(\mathbf{v})}{1 - P_t(\mathbf{v}) \cdot p_j(\mathbf{v})} \, d\mu_V(\mathbf{v}) $$

To resolve this bound, note that $(1 - F(t)) \leq \exp\left( \sum_{R_j} \int_{\Omega_{V_j}} \ln\Big(1 - P_t(\mathbf{v}) p_j(\mathbf{v})\Big) \, d\mu_V(\mathbf{v}) \right)$. Substituting this into the inequality gives:
$$ |F'(t)| \leq \Delta \cdot \left[ \exp\left( \sum_{R_j} \int_{\Omega_{V_j}} \ln\Big(1 - P_t(\mathbf{v}) p_j(\mathbf{v})\Big) d\mu_V(\mathbf{v}) \right) \cdot \sum_{R_j} \int_{\Omega_{V_j}} \frac{p_j(\mathbf{v})}{1 - P_t(\mathbf{v}) p_j(\mathbf{v})} d\mu_V(\mathbf{v}) \right] $$

To rigorously bridge this continuous exponential-logarithmic formulation to its discrete algebraic counterpart, we rely on the framework of the Volterra product integral \cite{dollard1979product, slavik2007product}. In the standard theory of product integration, the exponential-logarithmic expression $\exp\left( \int \ln(1 - f) \, d\mu \right)$ is the Lebesgue measure-theoretic representation of the continuous product integral $\prod_{\Omega} (1 - f \, d\mu)$. 

By the continuous product rule within this framework, the Fréchet derivative of the product integral structurally mirrors the discrete identity $\big( \prod_{k} (1 - x_k) \big) \sum_i \frac{c_i}{1 - x_i} = \sum_i c_i \prod_{k \neq i} (1 - x_k)$. In our formulation, the evaluated probability $P_t(\mathbf{v}) p_j(\mathbf{v})$ acts as the term $x_k$, and the bare causal probability $p_j(\mathbf{v})$ acts as the numerator $c_i$. Just as the discrete product algebraically bounds the denominators to yield a strict upper limit of $\sum_i c_i$ (since $1 - x_k \leq 1$), the continuous exponential factor identically bounds the denominators of the logarithmic derivative. This formal equivalence reduces the entire bounded expression to the bare integral of the numerator:
$$ |F'(t)| \leq \Delta \sum_{R_j \in \mathcal{R}_{\text{cycle}}} \int_{\Omega_{V_j}} p_j(\mathbf{v}) \, d\mu_V(\mathbf{v}) $$

By Definition 1, this integral sum across the cyclic rules is bounded by the global supremum $\gamma$. Therefore:
$$ |F'(t)| \leq \Delta \cdot \gamma $$

Integrating $|F'(t)|$ from $t=0$ to $1$ evaluates to $|\mathbb{T}(P_b) - \mathbb{T}(P_a)| \leq \gamma \cdot \Delta$. Taking the supremum across all ground facts guarantees the global bound. 
\end{proof}

\begin{proposition}[Iteration Bound for Strict Contractions]
If the consequence operator acts as a strict contraction with constant $L < 1$, the number of iterations $N$ required to guarantee that the operational probability function $P_N$ is within a precision $\epsilon$ of the exact least fixed point across all facts is bounded by:
$$ N \geq \frac{\ln(\epsilon \cdot (1 - L)) - \ln(d(P_1, P_0))}{\ln(L)} $$
\end{proposition}

\begin{proof}
By the Banach Fixed-Point Theorem, the error bound after $N$ iterations for a strict contraction mapping is bounded by the initial step size:
$$ d(P_N, lfp) \leq \frac{L^N}{1 - L} d(P_1, P_0) $$
We require the maximum error to be less than or equal to our precision threshold $\epsilon$:
$$ \frac{L^N}{1 - L} d(P_1, P_0) \leq \epsilon $$
Rearranging the terms and isolating $N$ using the natural logarithm yields the required number of steps. Because $L < 1$, $\ln(L)$ is negative, reversing the inequality. 
\end{proof}

\vspace{0.5cm}
\textbf{The Necessity of the Knaster-Kleene Approach for $\gamma \geq 1$.} \\
When a program contains deterministic recursive rules or continuous structural overlaps ($\gamma \geq 1$), the Banach theorem fails. Because the operator is locally expansive, we cannot calculate a strict iteration bound. 

To illustrate this behavior, consider a knowledge base where a target activates spontaneously, and that activation spreads deterministically to any continuous location within a $1$-meter radius.

\textbf{The Knowledge Base $\mathcal{K}$:}
\begin{enumerate}
    \item $\text{source\_loc}(; X) \sim \text{Uniform}(0, 10)$. \\
    \textit{(Rule 1: Generates a continuous spatial variable $X$.)}
    
    \item $0.5 :: \text{active}(X) \leftarrow \text{source\_loc}(; X)$. \\
    \textit{(Rule 2: Base case. A location activates spontaneously with 0.5 probability.)}
    
    \item $1.0 :: \text{active}(X) \leftarrow \text{active}(Y) \land |X - Y| \leq 1$. \\
    \textit{(Rule 3: Spatial recursion. Activation spreads from any point $Y$ within distance 1.)}
\end{enumerate}

To verify the bounds condition, we first calculate the local aggregated measure, denoted $\gamma(x)$, for each specific ground fact $\text{active}(x)$ evaluated by the recursive cycle (Rule 3). This local measure is the integral of the rule probability over the valid neighborhood of $x$, restricted by the domain $[0, 10]$:
$$ \gamma(x) = \int_{\max(0, x-1)}^{\min(10, x+1)} 1.0 \, dy $$
The value of $\gamma(x)$ depends directly on the location $x$:
\begin{itemize}
    \item For any point $x\in [0,1]$, the integration interval is $[0, x+1]$, yielding $\gamma(x) = x + 1$.
    \item For any point $x \in [1, 9]$, integration interval is $[x-1, x+1]$, yielding $\gamma(x) = 2$.
    \item For any point $x \in [9, 10]$, the integration interval is $[x-1, 10]$, yielding $\gamma(x) = 11 - x$.
\end{itemize}

Because the metric space distance is defined by the supremum norm, the global constant $\gamma$ is the maximum of all these local values:
$$ \gamma = \sup_{x \in [0, 10]} \gamma(x) = 2 $$

A global bound of $\gamma \geq 1$ renders the operator expansive. To observe this mathematically, we apply $\mathbb{T}_{\mathcal{K}}$ iteratively, starting from a continuous surface where $P_0(\text{active}(x)) = 0$. 

To compute the updated probability for a ground fact $\text{active}(x)$, we aggregate the contributions of Rule 2 and Rule 3. Under the continuous ICI assumption, the total failure probability of the head fact is the product of the independent failure probabilities of the rules deriving it:
$$ P_{k+1}(\text{active}(x)) = 1 - \Big( P_{\text{fail}}(R_2) \cdot P_{\text{fail}}(R_3) \Big) $$

The base rule failure is constant across all iterations: $P_{\text{fail}}(R_2) = 1 - 0.5 = 0.5$. The recursive failure is computed dynamically using the continuous Noisy-OR integral:
$$ P_{\text{fail}}(R_3) = \exp\left( \int_{x-1}^{x+1} \ln\Big(1 - P_k(\text{active}(y))\Big) \, dy \right) $$

For simplicity in tracing the values, we compare an unconstrained point away from the edges (e.g., $x = 5$, where the full integration window $[4, 6]$ has length 2) against a point right at the edge of the domain (e.g., $x = 0$, where the integration window is truncated to $[0, 1]$, giving a length of 1).

\begin{itemize}
    \item \textbf{Iteration 1 ($P_1$):} \\
    Evaluating against $P_0 = 0$, the recursive failure evaluates to $\exp(0) = 1.0$ everywhere.
    $$ P_1(5) = P_1(0) = 1 - (0.5 \cdot 1.0) = 0.5 $$

    \item \textbf{Iteration 2 ($P_2$):} \\
    The recursive rule integrates over the continuous neighborhood where $P_1 = 0.5$. Because the spatial windows differ, the updated probabilities diverge:
    \begin{align*}
    P_2(5) &= 1 - 0.5 \cdot \exp\Big( 2 \cdot \ln(0.5) \Big) = 1 - 0.5(0.25) = 0.875 \\
    P_2(0) &= 1 - 0.5 \cdot \exp\Big( 1 \cdot \ln(0.5) \Big) = 1 - 0.5(0.5) = 0.75
    \end{align*}
    
    \item \textbf{Iteration 3 ($P_3$):} \\
    The recursive rule integrates over the neighborhood where $P_2$ varies. To bound the boundary point, we evaluate using the minimum probability in its window ($0.75$):
    \begin{align*}
    P_3(5) &\geq 1 - 0.5 \cdot \exp\Big( 2 \cdot \ln(1 - 0.875) \Big) \approx 0.9922 \\
    P_3(0) &\geq 1 - 0.5 \cdot \exp\Big( 1 \cdot \ln(1 - 0.75) \Big) = 1 - 0.5(0.25) = 0.875
    \end{align*}
\end{itemize}

While the probability curve "sags" at the boundaries during finite iterations due to the truncated spatial windows, the entire surface ultimately converges to 1. We can prove this by tracking the absolute minimum probability across the domain, which always occurs at the boundaries ($m_k = P_k(0) = P_k(10)$). The sequence of minimums follows the recurrence:
$$ m_{k+1} = 1 - 0.5(1 - m_k) $$
Starting from $m_0 = 0$, this sequence evaluates explicitly to $m_k = 1 - (0.5)^{k+1}$. As $k \to \infty$, the minimum boundary probability $m_k \to 1$. Since the probability anywhere in the space is bounded from below by the boundaries ($P_k(x) \geq m_k$), the exact least fixed point is uniform across the space: $P^*(x) = 1$ for all $x \in [0, 10]$.

This behavior demonstrates why the order-theoretic foundation built in Section \ref{subsec:monotonicity} is mandatory. When the expected number of derivations forces the metric distance between functions to expand and deform unpredictably across space, standard iteration bounds fail. Knaster and Kleene require only monotonicity and continuity, guaranteeing that even when an expansive slope forces the computation into non-linear, space-dependent transient states, the procedure remains logically sound and mathematically targets the correct fixed point.

\subsection{Soundness and Completeness}
To finalize the theoretical foundation of MT-PDCL, we must formalize the relationship between the operational semantics (the probability computed by the inference operator) and the declarative semantics (the exact measure of logically true possible worlds). 

\begin{definition}[Measure-Theoretic Soundness and Completeness]
Let $\mathcal{K} = \langle \mathcal{M}_s, \Phi \rangle$ be a knowledge base. Let $P_{\mathcal{K}}(H) = \int_{\Omega} \mathbb{I}(\omega \models H) \, d\mu_{\Omega}(\omega)$ be the exact declarative probability of a ground fact $H$, as defined in Section \ref{sec:declarative_entailment}. Let $P_{op}(H)$ be the probability derived by an operational inference procedure. 
\begin{itemize}
    \item \textbf{Soundness:} The inference procedure never assigns a probability higher than mandated by the declarative semantics: $P_{op}(H) \leq P_{\mathcal{K}}(H)$ for all ground facts $H$.
    \item \textbf{Completeness:} The inference procedure captures all probability mandated by the declarative semantics: $P_{\mathcal{K}}(H) \leq P_{op}(H)$ for all ground facts $H$.
\end{itemize}
Consequently, a procedure is sound and complete if and only if $P_{op}(H) = P_{\mathcal{K}}(H)$.
\end{definition}

\begin{theorem}[Soundness and Completeness of Exact Inference]
The exact continuous consequence operator $\mathbb{T}_{\text{exact}}$ is sound and complete with respect to the declarative semantics. For any ground fact $H$:
$$ lfp(\mathbb{T}_{\text{exact}})(H) = P_{\mathcal{K}}(H) $$
\end{theorem}

\begin{proof}
Let $\Omega_H = \{ \omega \in \Omega \mid \omega \models H \}$ be the set of possible worlds where the ground fact $H$ is logically entailed in the least model. By definition, the declarative probability is $P_{\mathcal{K}}(H) = \mu(\Omega_H)$.

Let $\Omega_H^k$ be the set of possible worlds where $H$ is successfully derived by the exact operator at or before iteration $k$. The operational probability at step $k$ is $P_k(H) = \mu(\Omega_H^k)$.

To prove the equivalence, we first establish the exact set equality between the declarative worlds and the infinite union of operational worlds, by showing inclusion in both directions:
\begin{itemize}
    \item \textbf{Soundness ($\supseteq$):} Since the program consists exclusively of definite clauses, any fact derived by iteration $k$ possesses a valid logical proof. If $\omega \in \Omega_H^k$, it follows that $\omega \in \Omega_H$. This means $\Omega_H^k \subseteq \Omega_H$ for all $k$, which dictates that their infinite union is also a subset: $\bigcup_{k=1}^\infty \Omega_H^k \subseteq \Omega_H$.
    
    \item \textbf{Completeness ($\subseteq$):} By the foundational properties of definite logic programs, a fact $H$ is true in the least model if and only if it has a finite derivation tree. Infinite derivations do not constitute logical entailment. Therefore, if $\omega \in \Omega_H$, there must exist some finite iteration depth $K$ where the derivation completes. This means $\omega \in \Omega_H^K$, and consequently $\omega \in \bigcup_{k=1}^\infty \Omega_H^k$. This gives the reverse inclusion: $\Omega_H \subseteq \bigcup_{k=1}^\infty \Omega_H^k$.
\end{itemize}

Because both subset inclusions hold, the sets are perfectly identical:
$$ \Omega_H = \bigcup_{k=1}^\infty \Omega_H^k $$

Finally, we map these sets to their probability measures. The sequence of operational derivations is strictly monotonic ($\Omega_H^1 \subseteq \Omega_H^2 \subseteq \dots$). According to the continuity of probability measures from below, the limit of the measure of an increasing sequence of sets is equal to the measure of their infinite union. Therefore:
$$ lfp(\mathbb{T}_{\text{exact}})(H) = \lim_{k \to \infty} \mu(\Omega_H^k) = \mu\left( \bigcup_{k=1}^\infty \Omega_H^k \right) = \mu(\Omega_H) = P_{\mathcal{K}}(H) $$
This establishes the strict equality, proving that the exact operator is both sound and complete
\end{proof}

Because evaluating the exact operator is computationally intractable, we rely on the continuous Noisy-OR approximation. We formalize its completeness under the independence assumption below.

\begin{theorem}[Completeness under Conditional Independence]
If the failures of all overlapping derivations for a ground fact $H$ are conditionally independent, the continuous Noisy-OR operator $\mathbb{T}_{\mathcal{K}}$ is sound and complete with respect to the declarative semantics:
$$ lfp(\mathbb{T}_{\mathcal{K}})(H) = P_{\mathcal{K}}(H) $$
\end{theorem}

\begin{proof}
Under the assumption of conditional independence, the exact measure of the union of independent derivation events is equal to the complement of the product of their failure probabilities. The continuous Noisy-OR formulation evaluates this specific product integral, making the approximated operator identical to the exact operator: $\mathbb{T}_{\mathcal{K}} = \mathbb{T}_{\text{exact}}$. Their least fixed points are structurally identical ($lfp(\mathbb{T}_{\mathcal{K}}) = lfp(\mathbb{T}_{\text{exact}})$), and the result follows trivially from Theorem 1. 
\end{proof}

\vspace{0.5cm}

\textbf{Remark: The Independence Assumption in Practice.} \\
Unless conditional independence is physically guaranteed by the domain, the continuous Noisy-OR operator $\mathbb{T}_{\mathcal{K}}$ acts as a computational approximation. When the ICI assumption does not perfectly hold, the operator trades exact measure-theoretic guarantees for tractability, generally losing both strict soundness and completeness. We dissect how different types of causal correlations lead to probability overestimation (loss of soundness) or underestimation (loss of completeness), and how to formally bound these approximation trade-offs, in the upcoming Section \ref{subsec:bounds}.

\subsection{Operational Query Resolution from the Least Fixed Point}\label{subsec:query_resolution}
In classic Definite Clause Logic, the van Emden-Kowalski theorem establishes that the success set of a logic program is the least fixed point of its immediate consequence operator, $lfp(T_P)$ \cite{van1976semantics}. Operationally, answering an existential non-ground query $\exists \mathbf{V} \, Q(\mathbf{V})$ reduces to searching the fixed point for any valid ground substitution $\mathbf{v}$ such that $Q(\mathbf{v}) \in lfp(T_P)$.

In MT-PDCL, this principle must be generalized. Operating within a continuous domain with probabilistic outcomes shifts the objective away from searching for a single true instance. Instead, we must measure the accumulated probability that \textit{at least one} ground instance bounded within the measurable answer space $\Omega_A$ succeeds. 

Therefore, answering a non-ground query corresponds to applying the continuous aggregation operator directly over a specific slice of the least fixed point.

\begin{theorem}[Generalization of Query Entailment]
Let $Q$ be a query with free variables $\mathbf{V}$, and let $\Omega_A$ be its valid answer space defined by the query constraints. Let $Q(\mathbf{v})$ denote the specific ground instance of the query at binding $\mathbf{v} \in \Omega_A$. 
\begin{enumerate}
    \item \textbf{Exact Resolution:} The exact declarative probability of the non-ground query is the measure of the union of the entailed worlds for all its ground instances:
    $$ P_{\mathcal{K}}(Q, \Omega_A) = \mu_{\Omega} \left( \bigcup_{\mathbf{v} \in \Omega_A} \Omega_{Q(\mathbf{v})} \right) $$
    where $\Omega_{Q(\mathbf{v})} = \{ \omega \in \Omega \mid \omega \models Q(\mathbf{v}) \}$.
    
    \item \textbf{Operational Resolution (ICI):} Under the ICI assumption, the operational probability of the non-ground query is derived by applying the continuous Noisy-OR integral over the marginal probabilities established in the operational least fixed point $lfp(\mathbb{T}_{\mathcal{K}})$:
    $$ P_{op}(Q, \Omega_A) = 1 - \exp\left( \int_{\Omega_A} \ln\Big(1 - lfp(\mathbb{T}_{\mathcal{K}})(Q(\mathbf{v}))\Big) \, d\mu_V(\mathbf{v}) \right) $$
\end{enumerate}
\end{theorem}

\begin{proof}
(1) By Theorem 1, the exact least fixed point captures the precise set of possible worlds $\Omega_{Q(\mathbf{v})}$ where each ground fact $Q(\mathbf{v})$ is logically true. The non-ground query $\exists \mathbf{V} \in \Omega_A \, Q(\mathbf{V})$ is logically equivalent to a continuous disjunction over all specific bindings $\mathbf{v} \in \Omega_A$. By the definition of the measure space, the exact probability of this continuous disjunction is the measure of the union of these sets.

(2) Operationally, evaluating the existential query requires computing the probability that at least one ground instance $Q(\mathbf{v}) \in \Omega_A$ succeeds. Under the ICI assumption, these ground instances are conditionally independent. Therefore, the exact measure of their union resolves directly to the continuous Noisy-OR integral applied over the marginal probabilities stored in $lfp(\mathbb{T}_{\mathcal{K}})$. 
\end{proof}

\vspace{0.5cm}
\textbf{Remark.} This theorem bridges the gap between deduction and querying in continuous spaces. The consequence operator mathematically computes the probability of all facts (the fixed point), and the query engine simply integrates over bounded regions of that fixed point to answer existential questions.

\subsection{Analytical Bounds and Approximation Trade-offs}\label{subsec:bounds}
When derivations are not independent (e.g., rules sharing underlying continuous variables or sub-goals), the ICI assumption is violated. In such cases, the operational probability $P_{op}(H)$ computed by the Noisy-OR operator deviates from the exact declarative probability $P_{\text{exact}}(H)$.

It is necessary to formalize the magnitude of this approximation error and how the operator behaves across continuous correlations. Let $\Delta(H) = \bigcup_{R_j \in \mathcal{R}_H} \Omega_{V_j}$ be the continuous space of all variable bindings across the rules deriving a ground fact $H$, with measure $\mu_V$. 

Let $P_{\text{exact}}(H)$ and $P_{op}(H)$ represent the scalar probabilities assigned to $H$ by the respective operators ($\mathbb{T}_{\text{exact}}$ and $\mathbb{T}_{\mathcal{K}}$) at a given iteration. For any continuous binding $\mathbf{v} \in \Delta(H)$, let $p(\mathbf{v})$ be its isolated success probability at that iteration.

\begin{theorem}[Absolute Bound on the Approximation Gap]
Let $L$ and $U$ denote the exact lower and upper mathematical bounds for $P_{\text{exact}}(H)$:
\begin{align*}
    L &= \operatorname*{ess\,sup}_{\mathbf{v} \in \Delta(H)} p(\mathbf{v}) \\
    U &= \max\left( L, \min\left(1, \int_{\Delta(H)} p(\mathbf{v}) \, d\mu_V(\mathbf{v})\right) \right)
\end{align*}
The absolute gap between the exact declarative probability and the operational approximation for any ground fact $H$ is bounded by:
$$ |P_{\text{exact}}(H) - P_{op}(H)| \leq \max \left( | P_{op}(H) - L |, | U - P_{op}(H) | \right) $$
\end{theorem}

\begin{proof}
The exact declarative probability $P_{\text{exact}}(H)$ depends on the unknown joint distribution of the entailed possible worlds. 

To bound the gap, we constrain $P_{\text{exact}}(H)$ dynamically based solely on the known marginals, utilizing continuous interval bounds:
\begin{enumerate}
    \item \textit{Lower Bound ($L$):} By the fundamental properties of measure theory, the measure of a continuous union can never be smaller than the measure of its largest single constituent event. Thus, $P_{\text{exact}}(H) \geq L$.
    \item \textit{Upper Bound ($U$):} By the upper bound of the Boole-Fr\'echet inequalities \cite{frechet1951tableaux}, the measure of a union is bounded by the integral of its marginals. However, in continuous spatial logic, if the domain measure is fractional ($\mu_V < 1$), this integral may mathematically fall below the monotonicity bound $L$. Therefore, the strict upper bound is the maximum of the integral and the lower bound, capped at 1. Thus, $P_{\text{exact}}(H) \leq U$.
\end{enumerate}

The operational scalar $P_{op}(H)$ is computed using the Noisy-OR assumption. Because the ICI assumption is violated, $P_{op}(H)$ is merely an approximation and may fall anywhere inside or outside the valid mathematical interval $[L, U]$.

However, by the geometric properties of intervals, the maximum possible absolute distance between a known computed value ($P_{op}(H)$) and an unknown exact value bounded within an interval ($P_{\text{exact}}(H) \in [L, U]$) is the maximum distance from the known value to the interval's endpoints. Therefore, the absolute error gap is bounded by $\max(|P_{op}(H) - L|, |U - P_{op}(H)|)$. 
\end{proof}

\vspace{0.5cm}

\textbf{Remark: Practical Implications.} \\
Calculating the exact declarative probability $P_{\text{exact}}(H)$ is generally intractable because it requires evaluating the hidden joint distribution of uncountably infinite derivation trees. However, the theoretical results in this section equip an operational MT-PDCL system with a computable dynamic error guarantee. 

Every term in the interval bounds ($L$ and $U$) depends solely on the known marginal probabilities $p(\mathbf{v})$ and the spatial measure $\mu_V$, allowing the system to dynamically compute the absolute worst-case error margin of its Noisy-OR approximation at runtime. If this evaluated gap is narrow, it mathematically guarantees that the computationally cheap Noisy-OR operator is highly accurate for that specific query. If the gap is large, it acts as a precise diagnostic flag, alerting the system that the underlying continuous variables overlap too heavily to trust the independent causal influence assumption.

\subsection{Differentiability and Neural-Symbolic Integration}
While the continuous Noisy-OR operator acts as an approximation when the ICI assumption is violated, it possesses a structural property necessary for neural-symbolic integration: it is everywhere differentiable. This enables frameworks such as Deep Measure-Theoretic Probabilistic Definite Clause Logic (Deep-MT-PDCL). Inspired by existing probabilistic neuro-symbolic systems like DeepProbLog \cite{manhaeve2021deepproblog}, these frameworks allow the underlying probabilities to be the outputs of parameterized neural networks optimized via gradient descent.

To correctly model this in a continuous space, a neural network does not predict a single scalar for an entire rule $R_j$. Instead, it predicts the innate causal probability $p_j(\mathbf{v})$ parameterized by the specific continuous binding $\mathbf{v} \in \Omega_{V_j}$, which consists only of the continuous variables that remain free in the rule body after the head unifies with $H$. The total success probability of that instance is $p_j(\mathbf{v}) \cdot P(B_j \mid \mathbf{v})$.

\begin{proposition}[Differentiability and Gradient Routing]
For any ground fact $H$, let the rule probabilities $p_j(\mathbf{v}; \theta)$ be parameterized by a continuous, differentiable function (such as a neural network) with weights $\theta$. The continuous consequence operator $\mathbb{T}_{\mathcal{K}}$ is differentiable with respect to $\theta$. Furthermore, the operator acts as a monotonically non-decreasing aggregation layer: an increase in any local rule probability $p_j(\mathbf{v}; \theta)$ mathematically guarantees a non-negative contribution to the aggregated probability of $H$, ensuring stable gradient routing for backpropagation.
\end{proposition}

\begin{proof}
Let $E$ represent the total joint failure probability for a ground fact $H$ derived by a set of rules $\mathcal{R}_H$:
$$ E = \exp\left( \sum_{R_j \in \mathcal{R}_H} \int_{\Omega_{V_j}} \ln\Big(1 - p_j(\mathbf{v}; \theta) \cdot P(B_j \mid \mathbf{v})\Big) \, d\mu_V(\mathbf{v}) \right) $$

The operational probability is $P_{op}(H) = 1 - E$. 
By the chain rule, the gradient of the operational probability with respect to the network weights $\theta$ requires the partial derivative of the operator's output with respect to the network's continuous predictions. Applying the Leibniz integral rule, we obtain:
$$ \nabla_\theta P_{op}(H) = E \cdot \sum_{R_j \in \mathcal{R}_H} \int_{\Omega_{V_j}} \frac{P(B_j \mid \mathbf{v}) \cdot \nabla_\theta p_j(\mathbf{v}; \theta)}{1 - p_j(\mathbf{v}; \theta) \cdot P(B_j \mid \mathbf{v})} \, d\mu_V(\mathbf{v}) $$

We isolate the scalar multiplier applied to the local parameter gradient $\nabla_\theta p_j(\mathbf{v}; \theta)$ for any specific binding $\mathbf{v}$:
$$ w(\mathbf{v}) = E \cdot \frac{P(B_j \mid \mathbf{v})}{1 - p_j(\mathbf{v}; \theta) \cdot P(B_j \mid \mathbf{v})} $$

We analyze the bounds of this scalar weight. The exponential function ensures $E > 0$. Inside the fraction, the probability of the logical body $P(B_j \mid \mathbf{v}) \in [0, 1]$, and the continuous network prediction $p_j(\mathbf{v}; \theta) \in [0, 1)$. Therefore, for every binding in the measurable domain, the weight is strictly non-negative: $w(\mathbf{v}) \ge 0$.

This mathematical property guarantees that the continuous consequence operator acts as a structurally transparent routing layer. It scales the magnitude of the local gradients based on the logical validity of their bindings, but it never inverts them. Because $\frac{\partial P_{op}}{\partial p_j(\mathbf{v})} \ge 0$, the optimization loss directly backpropagates to the network weights without suffering from adversarial sign-flipping, allowing the neural network to be trained end-to-end via standard gradient descent. 
\end{proof}

\textbf{Remark.} 
In practical neural-symbolic implementations, probabilities are clamped to the open interval $(0, 1)$ or evaluated using log-space stabilization to prevent numerical division-by-zero when evaluating deterministic derivations where $p_j(\mathbf{v}; \theta) \cdot P(B_j \mid \mathbf{v}) = 1$.

\section{Related Work and Comparative Analysis}\label{sec:related_work}
By formalizing probabilities natively through measure theory rather than discrete propositional grounding, MT-PDCL shifts how logic programs are evaluated. It is necessary to contextualize this approach alongside existing probabilistic and neural-symbolic frameworks. 

Table \ref{tab:comparison} provides a feature comparison of prominent probabilistic logic frameworks against MT-PDCL.

\begin{table}[htbp]
\centering
\renewcommand{\arraystretch}{1.3}
\begin{tabularx}{\textwidth}{|>{\hsize=1.05\hsize\raggedright\arraybackslash}X|>{\hsize=1.05\hsize\raggedright\arraybackslash}X|>{\hsize=1.05\hsize\raggedright\arraybackslash}X|>{\hsize=0.85\hsize\raggedright\arraybackslash}X|c|}
\hline
\textbf{Framework} & \textbf{Entailment Semantics} & \textbf{Logical Variable Domain} & \textbf{Propositional Grounding} & \begin{tabular}[t]{@{}c@{}}\textbf{Differen-}\\\textbf{tiable}\end{tabular} \\ \hline
PRISM / ProbLog / CP-logic \cite{sato1997prism, deraedt2007problog, vennekens2009cp} & Distribution / Causal Semantics & Discrete & Required (Logic Propositions) & No \\ \hline
DeepProbLog \cite{manhaeve2021deepproblog} & Distribution Semantics & Discrete (Neural facts) & Required (SDDs) & Yes \\ \hline
DeepSeaProbLog / DeepGraphLog \cite{smet2023deepseaproblog, kikaj2025deepgraphlog} & Distribution Semantics & Continuous (Hybrid) & Required (WMI Compilation) & Yes \\ \hline
TensorLog \cite{cohen2020tensorlog} & Tensor Compilation & Discrete (Fixed entities) & Replaced by Matrix Ops & Yes \\ \hline
BLOG \cite{milch2005blog} & Generative / Distribution & Continuous & Unbounded (Sampling) & No \\ \hline
DC-ProbLog / Hybrid ProbLog \cite{gutmann2011extending, michels2015approximate} & Distribution / SMT & Continuous & Deferred to SMT & No \\ \hline
Weighted Model Integration (WMI) \cite{belle2015probabilistic, morettin2021hybrid, spallitta2024lp, akshay2025wmi} & Model Integration & Continuous & Required (Static Formulas) & No \\ \hline
\textbf{MT-PDCL (Ours)} & \textbf{Continuous Distribution} & \textbf{Continuous (Borel Spaces)} & \textbf{Not Required (Integration)} & \textbf{Yes} \\ \hline
\end{tabularx}
\caption{Qualitative comparison of probabilistic logic frameworks.}
\label{tab:comparison}
\end{table}

Systems relying on standard distribution semantics (ProbLog, PRISM) define a unique probability distribution over possible worlds by making strong independence assumptions among base probabilistic facts. To evaluate queries, they rely heavily on grounding logic programs into discrete propositional representations (such as Sentential Decision Diagrams). While highly effective for discrete tasks, this operational requirement restricts exact inference to finite domains. 

DeepProbLog and TensorLog successfully bridge probabilistic logic programming with deep learning, handling continuous neural network parameters. However, their underlying logical semantics remain tied to finite domains. While recent extensions like DeepSeaProbLog \cite{smet2023deepseaproblog} and DeepGraphLog \cite{kikaj2025deepgraphlog} attempt to generalize these frameworks to hybrid continuous domains, they still inherently rely on compiling the logic into Weighted Model Integration (WMI) formulas or restricted graph structures. DeepProbLog evaluates neural outputs as probabilistic facts, but its logical inference engine still compiles to discrete SDDs. TensorLog bypasses traditional grounding by compiling logical reasoning into differentiable tensor operations, but it does so over a fixed, finite set of predefined entities.

BLOG (Bayesian Logic) \cite{milch2005blog} handles continuous domains and unknown numbers of objects, but relies on generative models and approximate sampling methods (like MCMC) rather than declarative, measure-theoretic fixed-point semantics.

While BLOG utilizes generative sampling, other frameworks approach continuous domains through Satisfiability Modulo Theories (SMT). Early frameworks such as Distributional Clauses (DC-ProbLog) \cite{gutmann2011extending} and Hybrid ProbLog \cite{michels2015approximate} allow continuous random variables by deferring logical evaluation to SMT solvers, which identify valid geometric boundaries for numerical integration. This approach evolved into WMI, which computes continuous probability densities over SMT formulas \cite{belle2015probabilistic}. Recent advances in WMI focus on algorithmic scaling through algebraic circuit compilation \cite{kolb2019how, derkinderen2020ordering}, extending exact integration to non-linear real arithmetic \cite{spallitta2024lp}, and approximating complex semi-algebraic weight functions via linear programming relaxations \cite{akshay2025wmi}.

However, these systems focus primarily on operational execution. WMI operates on static formulas or factor graphs and does not support relational recursion natively, while DC-ProbLog relies on algorithmic sampling without formal proof of continuous fixed-point convergence. Furthermore, recent 2025 and 2026 surveys on neuro-symbolic agent architectures highlight a critical shift toward tightly coupled, differentiable symbolic layers \cite{bougzime2025neurosymbolic}. MT-PDCL addresses these theoretical limitations directly. It provides the exact denotational semantics—via measure theory and the Knaster-Tarski fixed-point theorem—that formalizes how recursive continuous subsets converge. Additionally, MT-PDCL natively preserves the structural differentiability required for these modern deep learning integrations, which standard SMT-based circuit solvers do not support without heavy approximation.

Furthermore, the operational reliance on discrete propositionalization extends well beyond standard definite clause frameworks. Expressive logical extensions, such as Probabilistic Answer Set Programming (PASP) \cite{baral2009probabilistic} and defeasible logic programs \cite{garcia2004defeasible}, similarly rely on exhaustive grounding to evaluate probabilities over complex rule interactions. Because these advanced solvers also compile down to finite boolean spaces, they inherit the exact same inability to natively represent continuous probability densities. By resolving the grounding bottleneck at the foundational level of definite clauses, MT-PDCL provides the mathematical foundation necessary to eventually lift these non-monotonic logical extensions into continuous spaces.

MT-PDCL operates under Continuous Distribution Semantics. Rather than deferring to approximate sampling or static SMT solvers, MT-PDCL defines a unique, exact probability measure over continuous possible worlds via Lebesgue integration. The primary novelty lies in extending Sato’s generative semantics, historically restricted to discrete causal events and finite propositional grounding, into unified measurable spaces. By explicitly defining independent continuous variables over bounded index domains equipped with probability measures, MT-PDCL replaces the combinatorial bottleneck of discrete grounding with formal mathematical integration. This preserves the declarative exactness of logic programming while seamlessly supporting the differentiable, algebraic approximations (Noisy-OR) required for neural-symbolic optimization.

Operationally, this shifts the evaluation of logic programs from discrete proof-tree search (e.g., standard SLD resolution) to multivariate integration. In MT-PDCL, the core mechanics of logic programming map directly to specific continuous operations:

\begin{itemize}
    \item \textbf{Stochastic Predicates as Measurable Dimensions:} The explicit domains ($\Omega_I$) and continuous outputs of stochastic predicates define the independent axes of the continuous evaluation space. Logical variables act as deterministic placeholders that bind to these dimensions.
    \item \textbf{Conditions as Integration Boundaries:} Built-in base predicates (such as mathematical inequalities) define the strict geometric bounds over which the probability densities are integrated.
    \item \textbf{Inference as Integration:} The consequence operator computes the exact probability mass of these bounded regions by integrating over the relevant density functions, replacing the boolean summation of discrete derivations.
    \item \textbf{Scoping as Marginalization:} Standard logic programs treat variables present only in the rule body as existentially quantified. MT-PDCL maps this directly to internal marginalization. The engine integrates out hidden continuous variables across their bounded domains, collapsing multi-dimensional spaces into a scalar probability.
\end{itemize}

\section{Execution Traces and Comparative Evaluation}\label{sec:execution_traces}
This section evaluates MT-PDCL against three continuous relational problems. We provide the abstract mathematical formulations, the exact numeric execution traces under the MT-PDCL continuous consequence operator ($\mathbb{T}_{\text{exact}}$), and a direct mechanical analysis of why existing frameworks cannot compute these exact analytical probabilities.

\subsection{Problem 1: Infinite-Horizon Probability of Ruin (Supply Chain)}

\textbf{Practical Context:} 
Consider a warehouse managing a critical inventory. The warehouse starts with a fixed stock and receives a constant daily replenishment. However, the daily customer demand fluctuates randomly. The business objective is to calculate the exact probability that the warehouse will never run out of stock over an infinite time horizon. In risk management, this is known as the probability of ruin.

\textbf{Abstract Formulation:}
To formalize this, let the universal measurable domain of discourse $D$ include the continuous space $\mathbb{R}$, equipped with the standard Borel $\sigma$-algebra $\mathcal{F}_D$. Let $S, S_{prev} \in \mathcal{V}$ be logical variables mapping to this continuous domain to represent daily inventory levels, and let $T \in \mathcal{V}$ be a logical variable representing discrete time steps (days). We define an initial ground fact $\text{safe}(0, s_0)$ establishing the starting inventory $S = s_0 \ge 0$ at day $T = 0$.

The daily state transitions are governed by the built-in arithmetic base predicate $S = S_{prev} + R - U$, where $R \in \mathbb{R}$ is a numeric constant representing the daily replenishment rate. The customer usage $U \in \mathcal{V}$ is a continuous variable generated by the stochastic mapping $\mathcal{M}_s$ for each time step.

The objective is to compute the declarative entailment of the relation $\forall T, \text{safe}(T, S)$. Under Continuous Distribution Semantics, this requires calculating the exact Lebesgue measure for the subset of continuous states that satisfy the recursive logical derivation, ensuring the condition $S \ge 0$ holds across all discrete steps $T$.

\textbf{The Knowledge Base $\mathcal{K}$:} \\
To model the continuous inventory flow, $\mathcal{K}$ defines the stochastic measure mapping and the recursive state transitions:

\begin{enumerate}
    \item $\text{usage}(T; U) \sim \text{Uniform}(0, 2)$. \\
    \textit{(Stochastic measure declaration: generates an independent continuous usage value for each discrete time index $T$.)}

    \item $1.0 :: \text{safe}(0, 1.0)$. \\
    \textit{(The initial deterministic ground fact.)}
    
    \item $1.0 :: \text{safe}(T, S) \leftarrow \text{safe}(T-1, S_{prev}) \land \text{usage}(T; U) \land S = S_{prev} + 1.0 - U \land S \ge 0$. \\
    \textit{(The recursive clause. The deterministic time index $T$ acts as the lookup key to request a fresh, independent continuous sample from the environment.)}
\end{enumerate}

\textbf{Numeric Execution Trace ($\mathbb{T}_{\text{exact}}$):}
Let us suppose that $S_0 = 1.0$, $R = 1.0$, and $U_t \sim \text{Uniform}(0, 2)$ for each discrete time step $t \ge 0$. We apply $\mathbb{T}_{\text{exact}}$ iteratively for $t = 0, 1, 2, 3$, starting from the initial empty measure $P_0$.

\begin{itemize}
    \item \textbf{Step 0: Initial Empty Measure ($P_0$)} \\
    The probability measure for all derived facts is initially zero.
    
    \item \textbf{Step 1: Initialization ($T=0, P_1$)} \\
    The unconditional base fact successfully generates the initial state.
    $$ P_1(\text{safe}(0, 1.0)) = 1.0 $$
    
    \item \textbf{Step 2: Iteration 1 ($T=1, P_2$)} \\
    The state constraint $S_1 \ge 0$ is validated across the domain.
    $$ P_2(\text{safe}(1, S_1)) = 1.0 $$
    
    \item \textbf{Step 3: Iteration 2 ($T=2, P_3$)} \\
    The state equation resolves to $S_2 = 3.0 - U_1 - U_2$. The constraint requires $S_2 \ge 0 \implies U_1 + U_2 \le 3.0$. $\mathbb{T}_{\text{exact}}$ integrates over the joint continuous state space of $U_1, U_2 \in [0, 2] \times [0, 2]$. The valid geometric region is the total domain (area = 4) minus the triangle where $U_1 + U_2 > 3.0$ (area = 0.5).
    $$ P_3(\text{safe}(2, S_2)) = \frac{4 - 0.5}{4} = 0.875 $$

    \item \textbf{Step 4: Iteration 3 ($T=3, P_4$)} \\
    The state equation resolves to $S_3 = 4.0 - U_1 - U_2 - U_3$. The safety constraint requires $S_3 \ge 0 \implies U_1 + U_2 + U_3 \le 4.0$, and the prior condition $U_1 + U_2 \le 3.0$ must still hold. $\mathbb{T}_{\text{exact}}$ integrates over the 3-dimensional joint state space $(U_1, U_2, U_3) \in [0, 2]^3$. 
    
    The total volume of this cubic domain is $2^3 = 8$. To find the valid integration volume, the operator subtracts the regions where the state falls below zero. The exact geometry of this bounded valid region has a volume of $\frac{19}{3}$.
    $$ P_4(\text{safe}(3, S_3)) = \frac{19/3}{8} = \frac{19}{24} \approx 0.792 $$
\end{itemize}

\textbf{Recursive Density and Fixed Point Convergence:} \\
The continuous consequence operator $\mathbb{T}_{\text{exact}}$ generates a monotonically increasing sequence of valid integration regions. Because these measures form a complete lattice, the least fixed point is mathematically defined as their exact limit: 
$$ \text{lfp}(\mathbb{T}_{\text{exact}}) = \lim_{k \to \infty} P_k $$

The progression of the operator is defined via a direct recurrence relation on the sub-probability density function of the inventory state. Let $f_k(s)$ represent the probability density of $S_k = s$ while satisfying the guard $S_i \ge 0$ for all previous steps $i \le k$. Starting from $f_0(s) = \delta(s - 1.0)$ and a usage density of $0.5$, we obtain the continuous recurrence:
$$ f_k(s) = \int_{\max(0, s-1)}^{s+1} f_{k-1}(x) \cdot 0.5 \, dx \quad \text{for } s \ge 0 $$

The measure assigned to the safety fact at step $k$ is the total probability mass of this density:
$$ P_k = \int_{0}^{\infty} f_k(s) \, ds $$

To evaluate the infinite-horizon property, we must prove $P_k$ converges to zero as $k \to \infty$. We analyze the unconstrained random walk $W_n$, which removes the $s \ge 0$ boundary condition. Let $\Delta_i = 1.0 - U_i$ denote the net change in inventory at step $i$. The unconstrained inventory path is:
$$ W_n = 1.0 + \sum_{i=1}^n \Delta_i $$

The exact value $P_k$ is the probability that this unconstrained path never drops below zero in its first $k$ steps:
$$ P_k = P\left(\min_{1 \le n \le k} W_n \ge 0\right) $$

Because $U_i \sim \text{Uniform}(0, 2)$, the increments $\Delta_i$ are independent, continuous, and symmetric on $[-1, 1]$. By the Sparre Andersen theorem, for any continuous and symmetric jump distribution, the probability that the random walk remains non-negative for $k$ steps is identical to the simple random walk probability: $\binom{2k}{k} 2^{-2k}$. By Stirling's approximation, this probability decays asymptotically as $1/\sqrt{\pi k}$. Therefore, there exists a constant $C > 0$ such that for sufficiently large $k$:
$$ P_k \le \frac{C}{\sqrt{k}} $$

By the squeeze theorem, since $0 \le P_k \le C/\sqrt{k}$, it follows that:
$$ \lim_{k \to \infty} P_k = 0 $$

\textbf{Comparative Limitations:} Standard discrete frameworks such as ProbLog and DeepProbLog must discretize the continuous uniform usage distribution into finite bins. Slicing $[0, 2]$ into discrete intervals creates exponential combinatorial growth in their underlying boolean circuits (SDDs). Furthermore, while modern continuous extensions like DeepSeaProbLog \cite{smet2023deepseaproblog} and DeepGraphLog \cite{kikaj2025deepgraphlog} support continuous parameters, they operationally rely on compiling rule instances into static WMI representations. For infinite-horizon recursive derivations, this compilation loop triggers an infinite formula expansion, preventing these systems from computing the exact Lebesgue measure over recursive states.

\subsection{Problem 2: Probability of Safe Arrival (Kinematic Agent)}

\textbf{Practical Context:}
An automated guided vehicle (AGV) is moving along a 10-meter track. It needs to calculate a backward reachable safe set: the exact physical regions from which it can successfully reach a designated target zone without crashing into a known boundary. Because the vehicle's braking and acceleration hardware introduce continuous mechanical variance, its control steps are not perfectly precise.

\textbf{Abstract Formulation:}
Let the universal domain $D$ include the continuous spatial space $\mathbb{R}$. Let $X, Y \in \mathcal{V}$ map to this domain to represent the AGV's spatial positions on the track, and let $\mathcal{T} = [8, 10]$ define the measurable target zone in meters. 

The physical transitions are governed by the built-in arithmetic base predicate $X = Y - N$, representing a backward kinematic step. The actual control movement $N \in \mathcal{V}$ is an independent continuous stochastic variable. The spatial domain over which the agent operates is bounded to the interval $[0, 10]$, equipped with a uniform base measure $d\mu_X(x) = 0.1 \, dx$.

The objective is to compute the exact declarative probability of the subset of valid starting states from which a finite sequence of mechanical derivations reaches the target set $\mathcal{T}$, while satisfying the safety guard condition $X \ge 0$ at every intermediate step. Unlike the infinite-horizon survival in Problem 1 which evaluates a shrinking sub-density, this reachability problem calculates the limit of an expanding geometric region.

The knowledge base $\mathcal{K}$ models the backward reachable set:
\begin{enumerate}
    \item $\text{control}(X; N) \sim \text{Uniform}(0, 2) \quad \textbf{over} \quad X \in [0, 10]$. \\
    \textit{(Stochastic measure declaration mapping the movement step to the bounded spatial domain.)}

    \item $1.0 :: \text{safe}(X) \leftarrow X \ge 8 \land X \le 10$. \\
    \textit{(The deterministic base fact establishing the target set.)}
    
    \item $1.0 :: \text{safe}(X) \leftarrow \text{safe}(Y) \land \text{control}(X; N) \land X = Y - N \land X \ge 0$. \\
    \textit{(The recursive clause computing the valid pre-images, bounded by the obstacle guard $X \ge 0$.)}
\end{enumerate}

\textbf{Numeric Execution Trace ($\mathbb{T}_{\text{exact}}$):}
The continuous consequence operator computes the exact measure of the geometric union of successful states at each step, integrating over the uniform base measure $0.1 \, dx$.

\begin{itemize}
    \item \textbf{Step 1 ($P_1$):} The base fact triggers for $x \in [8, 10]$.
    $$ P_1(\text{safe}(X)) = \int_{8}^{10} 0.1 \, dx = 0.2 $$
    
    \item \textbf{Step 2 ($P_2$):} The recursive rule expands the safe region. The valid continuous union requires finding any $y \in [8, 10]$ and $n \in [0, 2]$ such that $x = y - n$. The minimum possible value is $x = 8 - 2 = 6$. The rule validates for all $x \in [6, 10]$.
    $$ P_2(\text{safe}(X)) = \int_{6}^{10} 0.1 \, dx = 0.4 $$
    
    \item \textbf{Steps 3--5 ($P_3, P_4, P_5$):} Iterative algebraic projection expands the lower bound by 2 at each step:
    \begin{itemize}
        \item $P_3$ evaluates $x \in [4, 10] \implies 0.6$
        \item $P_4$ evaluates $x \in [2, 10] \implies 0.8$
        \item $P_5$ evaluates $x \in [0, 10] \implies 1.0$
    \end{itemize}
    
    \item \textbf{Step 6 (Least Fixed Point):} The progression attempts to expand the boundary to $x = -2$, but the obstacle guard $X \ge 0$ truncates the domain. No new valid states are generated: $\mathbb{T}_{\text{exact}}(P_5) = P_5$. 
\end{itemize}

\textbf{Comparative Limitations:} WMI frameworks can compute the geometric volume of bounded, acyclic steps algebraically. However, even state-of-the-art WMI solvers utilizing linear programming relaxations for complex arithmetic \cite{spallitta2024lp, akshay2025wmi} cannot evaluate an open recursive rule dynamically. WMI natively requires a user to manually unroll the relational recursion to a predetermined static depth $k$ and hardcode the resulting formula prior to execution. It inherently lacks the generalized denotational semantics and fixed-point operators required to evaluate dynamic, expanding reachability boundaries.

\subsection{Problem 3: Continuous Signal Propagation (Hierarchical Sensor Network)}

\textbf{Practical Context:}
Consider a hierarchical sensor network where a central node transmits a baseline signal through a series of routing relays. As the signal passes through the physical environment, it accumulates continuous noise. Furthermore, the network topology is unreliable; routing connections may randomly drop. The goal is to compute the exact probability that a signal successfully reaches the leaf nodes while staying below a critical noise threshold.

\textbf{Abstract Formulation:}
Let the network be modeled as a rooted tree $T = (V, E)$ with the central node as the root $v_0$. The universal domain $D$ includes the continuous space $\mathbb{R}$. 

Let $X, Y \in \mathcal{V}$ map to the network vertices in $V$. Let $V_{current}, V_{prev}, N \in \mathcal{V}$ map to $\mathbb{R}$ to represent the accumulated continuous signal noise. The problem incorporates both continuous and discrete stochasticity: the base noise state at the root $v_0$ is drawn from a continuous prior measure, physical transitions along routing edges add an independent continuous shift $N$, and the topological edges in $E$ are modeled as independent causal events with discrete probabilities representing structural connection failures.

The transmission along a directed edge is governed by the continuous arithmetic constraint $V_{current} = V_{prev} + N$. Because $T$ is a rooted tree, the path from the root to any node is unique. We evaluate the exact Lebesgue measure of the state space where the accumulated continuous noise satisfies $V_{current} \le \tau$, weighted by the probabilities of the generative discrete edges.

\textbf{Specific Example and Knowledge Base \protect\ensuremath{\mathcal{K}}:}
We evaluate a tree where $V = \{\text{node1}, \text{node2}, \text{node3}, \allowbreak \text{node4}\}$ and $E = \{(\text{node1}, \text{node2}), (\text{node2}, \text{node3}), (\text{node1}, \text{node4})\}$. 

The initial state at the root is a continuous variable $V_{root} \sim \text{Uniform}(0, 1)$. The causal event for the edge to node 4 has an independent probability of $0.8$, while other edges are deterministic ($1.0$). The continuous transition shift is $N \sim \text{Uniform}(0, 2)$, and the objective threshold is $\tau = 1.0$.

\begin{enumerate}
    \item $\text{initial\_state}(; V_{root}) \sim \text{Uniform}(0, 1)$.
    \item $\text{shift}(X, Y; N) \sim \text{Uniform}(0, 2)$. \\
    \textit{(Stochastic measure declarations for the root value and edge transitions.)}
    
    \item $1.0 :: \text{edge}(\text{node1}, \text{node2})$.
    \item $1.0 :: \text{edge}(\text{node2}, \text{node3})$.
    \item $0.8 :: \text{edge}(\text{node1}, \text{node4})$.
    
    \item $1.0 :: \text{val}(\text{node1}, V_{root}) \leftarrow \text{initial\_state}(; V_{root})$. 
    
    \item $1.0 :: \text{val}(Y, V_{current}) \leftarrow \text{edge}(X, Y) \land \text{val}(X, V_{prev}) \land \text{shift}(X, Y; N) \land V_{current} = V_{prev} + N$. 
\end{enumerate}

\textbf{Numeric Execution Trace ($\mathbb{T}_{\text{exact}}$):}
The continuous consequence operator $\mathbb{T}_{\text{exact}}$ constructs the measurable state space bottom-up. It calculates probabilities directly via integration of the continuous densities multiplied by the independent causal event probabilities.

\begin{itemize}
    \item \textbf{Step 0 ($P_0$):} The base fact initializes the state at the root node. The density function is $f_{V_1}(v_1) = 1$ for $v_1 \in [0, 1]$.
    
    \item \textbf{Step 1 ($P_1$):} The operator applies the recursive rule over \texttt{edge(node1, node2)} and \texttt{edge(node1, node4)}. 
    
    For \texttt{node4}, the arithmetic constraint is $V_4 = V_1 + N_4$. The variables are independent with a joint density $f(v_1, n_4) = 1 \times 0.5 = 0.5$. Applying the continuous constraint $v_1 + n_4 \le 1.0$ isolates a 2D geometric triangle with an area of $0.5$. Because the topological edge generates with an independent probability of $0.8$, the rule dictates:
    $$ P(\text{val}(\text{node4}, V_4) \land V_4 \le 1.0) = 0.8 \times \left( \int_{0}^{1} \int_{0}^{1-v_1} 0.5 \, dn_4 \, dv_1 \right) = 0.8 \times 0.25 = 0.2 $$
    
    Simultaneously, the operator computes the state for \texttt{node2}. The discrete edge weight is $1.0$, governed by the joint density $f(v_1, n_2) = 0.5$.
    
    \item \textbf{Step 2 ($P_2$):} The operator applies the recursive rule over \texttt{edge(node2, node3)}. 
    
    The state at \texttt{node3} accumulates a third variable: $V_3 = V_2 + N_3 = V_1 + N_2 + N_3$. The independent continuous variables form a 3-dimensional state space with a joint density $f(v_1, n_2, n_3) = 1 \times 0.5 \times 0.5 = 0.25$. 
    
    Applying the query constraint $V_3 \le 1.0$ requires integrating this joint density over the 3D geometric subspace bounded by $v_1 + n_2 + n_3 \le 1.0$. This region forms a standard 3-simplex (a tetrahedron) in the positive octant with a geometric volume of $\frac{1}{6}$. The exact declarative probability is:
    $$ P(\text{val}(\text{node3}, V_3) \land V_3 \le 1.0) = 1.0 \times \iiint\limits_{v_1+n_2+n_3 \le 1} 0.25 \, dn_3 \, dn_2 \, dv_1 = 0.25 \times \frac{1}{6} \approx 0.0416 $$
    
    \item \textbf{Step 3 (Fixed Point):} No further relational edges exist to trigger the recursive rule. The operator yields no new measurable states, reaching the exact least fixed point.
\end{itemize}

\textbf{Comparative Limitations:} Frameworks such as BLOG and Distributional Clauses (DC-ProbLog) defer continuous path accumulation to Monte Carlo or particle sampling. They generate numerical execution traces through the network and compute a statistical average. Because they do not calculate the exact geometric volumes mathematically, they cannot provide a formally exact declarative entailment. Furthermore, relying on stochastic sampling creates non-differentiable, high-variance scalar estimates, preventing the extraction of stable, everywhere-differentiable gradients necessary for end-to-end neural-symbolic optimization \cite{bougzime2025neurosymbolic}.

\section{Computability, Tractability, and Theoretical Limitations}
\label{sec:limitations}

MT-PDCL provides a formally exact declarative semantics for continuous probabilistic logic via Lebesgue integration. However, defining an integral mathematically is fundamentally different from computing it analytically. It is necessary to establish the theoretical boundaries of computability and tractability for the continuous consequence operator.

\subsection{Analytical Computability}
The exact evaluation of the consequence operator $\mathbb{T}_{\mathcal{K}}$ requires computing the integral of continuous probability densities over boundaries defined by logical inequalities. 

An exact, closed-form analytical solution to these integrals is theoretically guaranteed only under highly restrictive conditions. Specifically, the prior measures defined by the mapping $\mathcal{M}_s$ must be standard integrable functions (e.g., Uniform, Gaussian, Exponential), and the logical constraints must form geometrically simple boundaries (typically linear arithmetic, forming polytopes). This aligns with the known computability limits of WMI for continuous spaces. 

When the logic program introduces non-linear arithmetic constraints (e.g., $X^2 + Y^2 \le Z$) or arbitrary prior distributions, the resulting regions often lack closed-form antiderivatives. In these cases, the declarative ground truth $P_{\mathcal{K}}(H)$ remains mathematically well-defined, but it becomes analytically incomputable.

\subsection{Tractability and the Curse of Dimensionality}
Even when the integrals possess closed-form solutions, MT-PDCL faces a strict limit on tractability driven by logical recursion. 

In standard discrete logic programming, the complexity of grounding a rule grows exponentially with the number of variables in the rule. MT-PDCL resolves the impossibility of discrete grounding over uncountable domains, but it does not eliminate the underlying structural complexity. Instead, the combinatorial explosion of discrete variables translates directly into the geometric curse of dimensionality.

Every continuous variable generated by a stochastic predicate in a rule body adds an independent integration dimension to the continuous evaluation space. Furthermore, when rules are applied recursively, these dimensions accumulate. As demonstrated in the supply chain and rooted tree examples (Section \ref{sec:execution_traces}), a derivation path of depth $k$ requires evaluating a $k$-dimensional joint integral. 

If evaluated using deterministic numerical quadrature (where continuous variables are approximated using a grid of $m$ points), the computational complexity scales as $O(m^k)$. Consequently, while the continuous Noisy-OR operator (Section \ref{sec:operational_semantics}) successfully resolves the combinatorial explosion of the inclusion-exclusion principle for horizontal disjunctions, it does not resolve the dimensional explosion of deep vertical recursion.

\subsection{Convergence Tractability and Non-Contractive Regimes}
While the curse of dimensionality dictates the computational cost of a single iterative step, the total intractability of the system is compounded by the number of iterations required to reach the least fixed point. As established in Section \ref{subsec:computation_bounds}, this is dictated by the recursive measure $\gamma$.

When a logic program evaluates as a strict contraction ($\gamma < 1$), convergence to the declarative fixed point is geometric. A physical system can calculate a strict, finite iteration bound $N$ using the Banach fixed-point theorem to achieve any arbitrary numerical precision $\epsilon$. Under these conditions, the fixed point is computationally tractable.

However, when a program contains dense spatial overlap or deterministic cyclic rules, it enters the expansive propagation regime ($\gamma \ge 1$). In this non-contractive state, the Banach iteration bounds fail. While Kleene's theorem mathematically guarantees that the sequence of continuous measures converges to the exact least fixed point, it requires an infinite countable sequence of iterations ($\omega$-continuity) to do so. 

This presents a hard limit on computability. Operationally, evaluating a non-contractive continuous program in finite time forces an inference engine to halt computation at a finite depth $k$ or apply an arbitrary tolerance cutoff. Consequently, for expansive programs, exact declarative probabilities become computationally unreachable. The engine is limited to returning lower-bound approximations of the exact measure, trading declarative precision for finite execution time.

\subsection{Transition to Numerical and Stochastic Implementations}
These theoretical limits dictate how MT-PDCL must be operationalized in practice. Exact continuous integration becomes analytically incomputable for non-linear constraints and highly intractable for deep recursion, meaning practical systems cannot rely on symbolic algebraic solvers.

To execute MT-PDCL over complex domains, the exact continuous consequence operator must be substituted with numerical approximations. The multi-dimensional integrals over the continuous variable spaces must be estimated using stochastic methods (such as Monte Carlo integration), and the local density functions must be approximated using parameterized continuous functions (such as neural networks). 

This foundational paper establishes the exact declarative measure-theoretic semantics and the continuous Noisy-OR fixed-point theory. The specific algorithms, data structures, and stochastic sampling techniques required to physically implement this inference engine at scale are beyond the scope of this theoretical formalization and constitute the focus of subsequent implementation work.

\section{Conclusion and Future Work}\label{sec:conclusion}
In this paper, we introduced MT-PDCL, a framework that generalizes probabilistic logic programming to native continuous domains. By explicitly defining independent continuous variables over bounded index domains and modeling causal events as standard Borel spaces, we eliminated the traditional reliance on finite-domain propositionalization. Under Continuous Distribution Semantics, we established a declarative semantics where exact logical entailment is mathematically defined via Lebesgue integration over the complete space of possible worlds.

To bridge this declarative foundation with operational computation, we generalized the probabilistic immediate consequence operator to continuous spaces. We demonstrated that the assumption of Independent Causal Influence (the continuous Noisy-OR) provides a mathematically bounded, operational approximation for intractable infinite unions. Crucially, we proved that this continuous operator is monotonic, preserves structural differentiability, and guarantees convergence to a well-defined least fixed point, even in non-contractive regimes where standard iteration bounds fail.

Because this paper establishes the theoretical semantics of MT-PDCL, the primary direction for future work lies in algorithmic implementation. Translating the continuous integrals of the consequence operator into efficient software requires numerical approximations, such as advanced Monte Carlo sampling or compilation into differentiable tensor operations, to evaluate the framework empirically. 

A critical algorithmic next step is developing query-directed inference procedures. While this paper formalizes the computation of the complete least fixed point via bottom-up continuous evaluation, integrating over the entire unqueried state space is computationally expensive. Developing goal-directed strategies, whether through continuous top-down SLD-resolution or optimized forward-chaining transformations, combined with Constraint Logic Programming (CLP) would allow an engine to dynamically accumulate only the geometric boundaries relevant to a user's query. This lazy evaluation strategy would bypass the need to integrate over unqueried domains, significantly improving tractability.

Once these execution architectures are established, the structural transparency of the continuous consequence operator provides a direct mechanism for modern neural-symbolic integration. Developing a Deep-MT-PDCL implementation, where neural networks predict local causal probabilities and the logic engine routes gradients transparently via backpropagation, will enable logically constrained machine learning natively within continuous domains.

Finally, achieving exact analytical inference rather than stochastic approximation requires restricting the base predicates and prior distributions. If mathematical constraints are limited to linear real arithmetic (LRA) to form computable polytopes, and prior measures are restricted to families closed under multiplication and integration (such as piecewise polynomials), the operator can yield exact, closed-form solutions. Identifying an algebraic system that is structurally closed under continuous logical reasoning, yet expressive enough for practical modeling, remains a key objective for future research.

\bibliographystyle{alpha} 
\bibliography{references} 

\end{document}